\documentclass{article}

\usepackage[nonatbib,final]{neurips_2020}

\usepackage[utf8]{inputenc} 
\usepackage[T1]{fontenc}    
\usepackage{hyperref}       
\usepackage{url}            
\usepackage{booktabs}       
\usepackage{amsfonts}       
\usepackage{nicefrac}       
\usepackage{microtype}      
\usepackage{comment}

\usepackage{enumerate}
\usepackage{amsmath}
\usepackage{amssymb}
\usepackage{subfigure}
\usepackage{mathrsfs}
\usepackage{booktabs}
\usepackage{multirow}
\usepackage{array}

\usepackage{times}
\usepackage{epsfig}
\usepackage{amsmath}
\usepackage{amssymb}
\usepackage{mathrsfs}
\usepackage{multirow}
\usepackage{array}
\usepackage{caption}
\usepackage{placeins} 

\usepackage{color} 
\usepackage{mdwlist}

\usepackage{algorithm}  
\usepackage{algpseudocode}

\usepackage{wrapfig,lipsum,booktabs}
\newcommand{\R}{\mathbb{R}}
\usepackage{mathrsfs,amsthm,amsmath,amssymb}
\newtheorem{theorem}{Theorem}

\newtheorem{lemma}{Lemma}

\newcommand{\PreserveBackslash}[1]{\let\temp=\\#1\let\\=\temp}
\newcolumntype{C}[1]{>{\PreserveBackslash\centering}p{#1}}
\newcolumntype{R}[1]{>{\PreserveBackslash\raggedleft}p{#1}}
\newcolumntype{L}[1]{>{\PreserveBackslash\raggedright}p{#1}}

\usepackage{comment}

\title{A Dictionary Approach to Domain-Invariant Learning in Deep Networks}

\author{%
	Ze Wang\\
	Purdue University\\
	\texttt{zewang@purdue.edu}\\
 \And
 	Xiuyuan Cheng\\
 	Duke University\\
 	\texttt{xiuyuan.cheng@duke.edu} \\
 \And
   Guillermo Sapiro\\
   Duke University\\
   \texttt{guillermo.sapiro@duke.edu} \\
 \And
  Qiang Qiu\\
  Purdue University\\
\texttt{qqiu@purdue.edu} \\
 }
\begin{document}

\maketitle

\begin{abstract}

In this paper, we consider domain-invariant deep learning by explicitly modeling domain shifts with only a small amount of domain-specific parameters in a Convolutional Neural Network (CNN). 
By exploiting the observation that a convolutional filter can be well approximated as a linear combination of a small set of dictionary atoms, we show for the first time, both empirically and theoretically, that domain shifts can be effectively handled by decomposing a convolutional layer into a domain-specific atom layer and a domain-shared coefficient layer, while both remain convolutional. 
An input channel will now first convolve spatially only with each respective domain-specific dictionary atom to ``absorb" domain variations, and then output channels are linearly combined using common decomposition coefficients trained to promote shared semantics across domains. 
We use toy examples, rigorous analysis, and real-world examples with diverse datasets and architectures, to show the proposed plug-in framework's effectiveness in cross and joint domain performance and domain adaptation. 
With the proposed architecture, we need only a small set of dictionary atoms to model each additional domain, which brings a negligible amount of additional parameters, typically a few hundred.

\end{abstract}

\section{Introduction}

Training supervised deep networks requires large amount of labeled training data;
however, well-trained deep networks often degrade dramatically on testing data from a significantly different domain. 
In real-world scenarios, such domain shifts are introduced by many factors, such as different illumination, viewing angles, and resolutions.
Research topics such as transfer learning and domain adaptation are studied to promote invariant representations across domains with different levels of availabilities of annotated data.

Recent efforts on learning cross-domain invariant representations using deep networks generally fall into two categories.
The first one is to learn a common network with constraints encouraging invariant feature representations across different domains \cite{long2015learning,long2016unsupervised,tzeng2014deep}. The feature invariance is usually measured by feature statistics like maximum mean discrepancy, or feature discriminators using adversarial training \cite{ganin2016domain}. 
While these methods introduce no additional model parameters, the effectiveness largely depends on the degree of domain shifts.
The other direction is to explicitly model domain specific characteristics with a multi-stream network structure where different domains are modeled by corresponding sub-networks at the cost of extra parameters and computations \cite{rozantsev2018residual}. 

In this paper, we model domain shifts through domain-adaptive filter decomposition (DAFD) with layer branching.
At a branched layer, we decompose each filter over a small set of domain-specific dictionary atoms to model intrinsic domain characteristics, while enforcing shared cross-domain decomposition coefficients to align invariant semantics. 
A regular convolution is now decomposed into two steps. First, a domain-specific dictionary atom convolves spatially only each individual input channel for shift  ``correction.'' Second, the ``corrected" output channels are weighted summed using domain-shared decomposition coefficients (1$\times$1 convolution) to promote common semantics. 
When domain shifts happen in space, we rigorously prove that such layer-wise ``correction'' by the same spatial transform applied to atoms suffices to align the learned features, contributing to the needed theoretical foundations in the field. 

Comparing to the existing subnetwork-based methods, the proposed method has several appealing properties: 
First, only a very small amount of additional trainable parameters are introduced to explicitly model each domain, i.e., domain-specific atoms. 
The majority of the parameters in the network remain shared across domains, and learned from abundant training data to effectively avoid overfitting.
Furthermore, the decomposed filters reduce the overall computations significantly compared to previous works, where computation typically grows linearly with the number of domains.

We conduct extensive real-world face recognition (with domain shifts and simultaneous multi-domains inputs), image classification, and segmentation experiments, and observe that, with the proposed method, invariant representations and performance across domains are consistently achieved without compromising the performance of individual domain. 

Our main contributions are summarized as follows:

\begin{itemize*}
	\item We propose plug-in domain-invariant representation learning through filter decomposition with layer branching, where domain-specific atoms are learned to counter domain shifts, and semantic alignments are enforced with cross-domain common decomposition coefficients. 
	
	\item We both theoretically prove, contributing the much needed foundations in CNN-based invariant learning, and empirically demonstrate that by stacking the atom-decomposed branched layer, invariant representations across domains are achieved progressively.
	
	\item The majority of network parameters remain shared across domains, which alleviates the demand for massive annotated data from every domain, and introduces only a small amount of additional computation and parameter overhead.
	Thus the proposed approach serves as an efficient way for domain invariant learning and its applications to domain shifts and simultaneous multi-domain tasks.
.
	
\end{itemize*}

\begin{figure*}[h]
	\centering
	\vspace{-5mm}
	\subfigure[Regular CNN]{
		\includegraphics[width=0.3\linewidth]{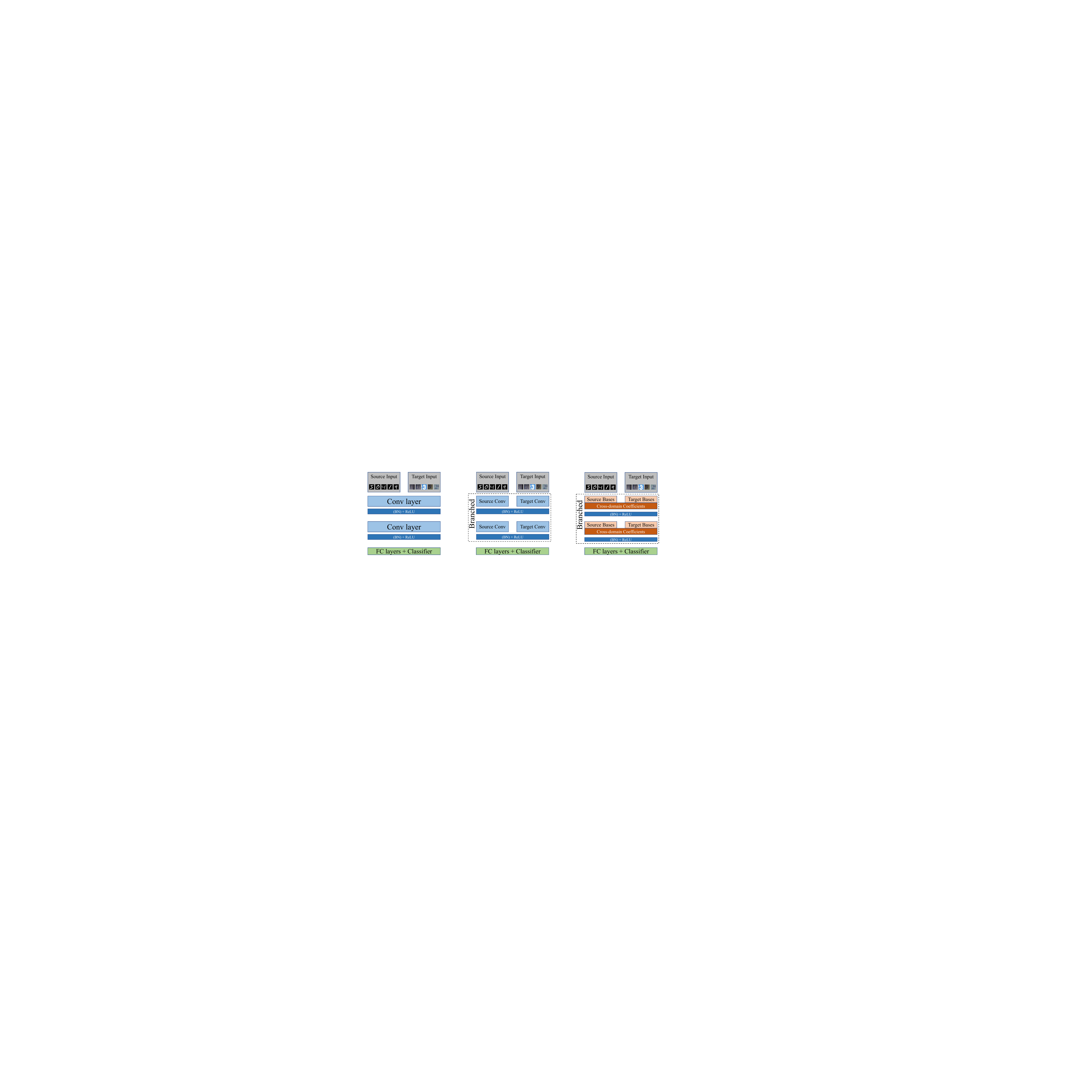}
		\label{fig:a0}
	}
	\subfigure[Basic Branching]{
		\includegraphics[width=0.3\linewidth]{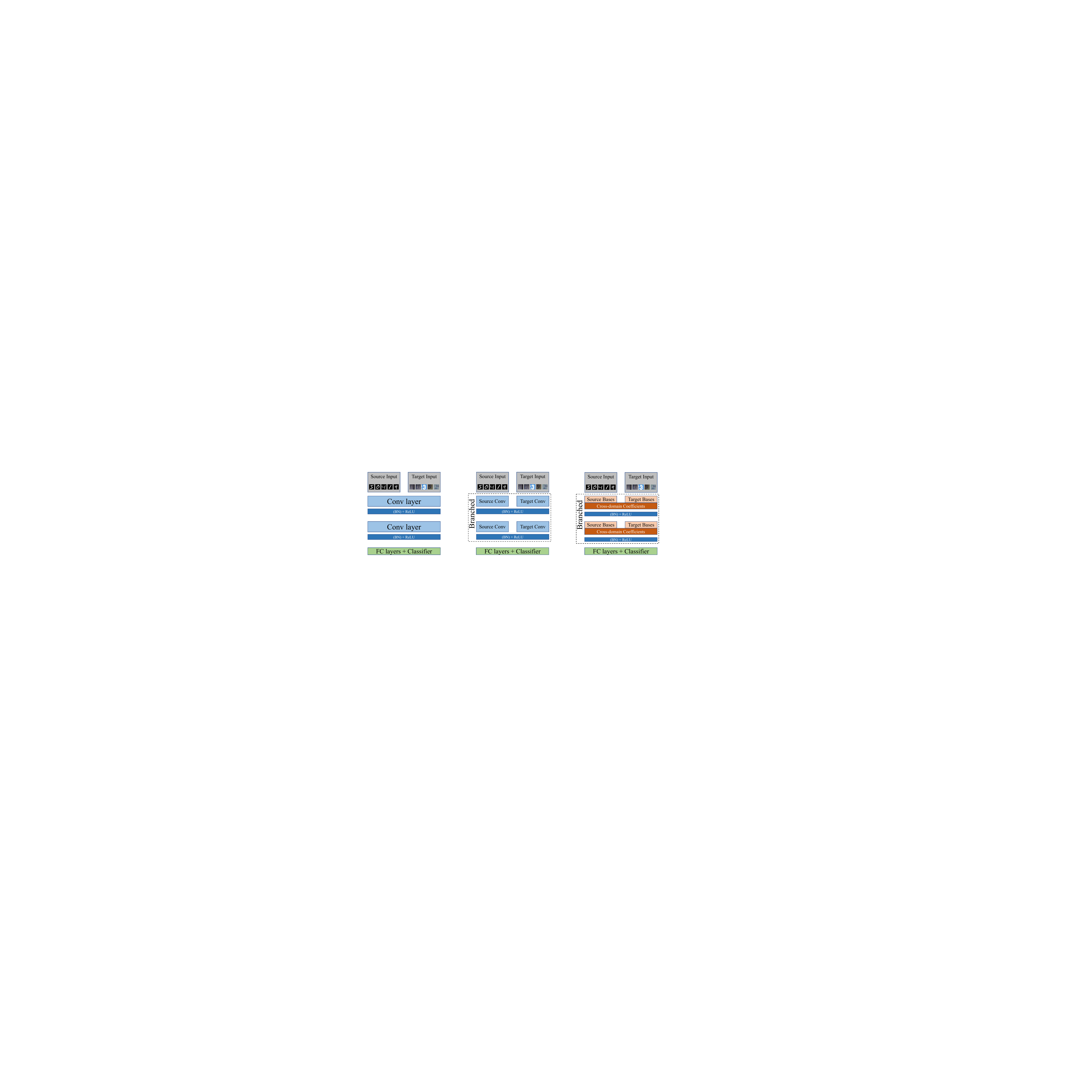}
		\label{fig:a1}
	}
	\subfigure[Branching with DAFD]{
		\includegraphics[width=0.28\linewidth]{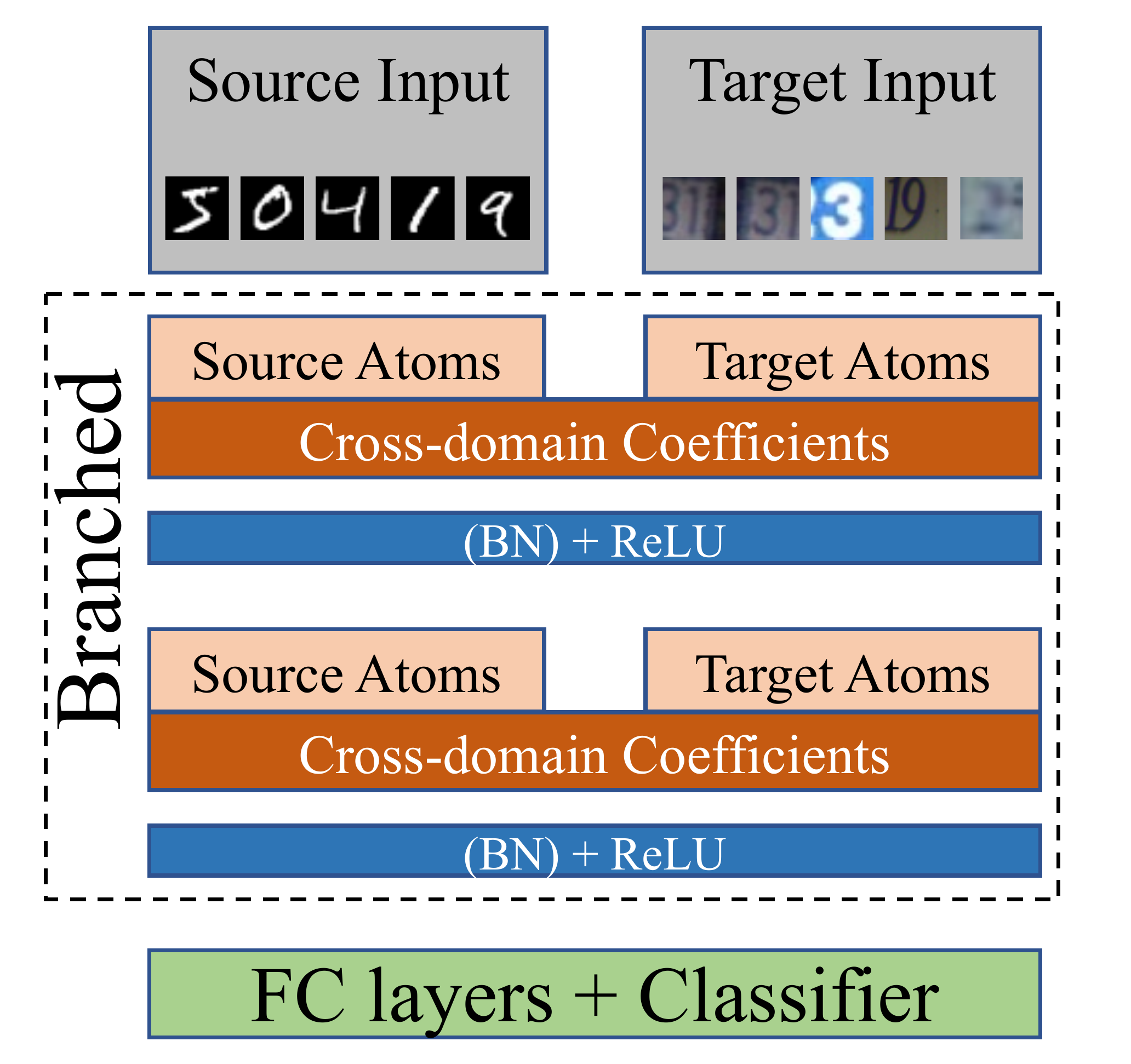}
		\label{fig:a2}
	}
	\caption{Three candidate architectures considered for domain-invariant representation learning. In (a), a set of common network parameters are trained to model both source and target domains. In (b), the domain characteristics are explicitly modeled by two sets of convolutional filters in each convolutional layer. Our approach is illustrated in (c) where domain-adaptive atoms are learned to ``absorb" domain shifts, while the decomposition coefficients are shared across domains to promote and exploit common semantics.
	}
	\label{fig:main}
\end{figure*}

\section{Domain-adaptive Filter Decomposition for Invariant Learning}
\label{app}

A straightforward way to address domain shifts is to learn from multi-domain training data a single network
as in Figure~\ref{fig:a0}. 
However, the lack of explicitly modelling of individual domains often results in unnecessary information loss and performance degradation as discussed in \cite{rozantsev2018residual}. 
Thus, we often simultaneously observe underfitting for domains with abundant training data, and overfitting for domains with limited training.
In this section, we start with a simplistic pedagogical formulation, domain-adaptive layer branching as in Figure~\ref{fig:a1}, where domain shifts are modeled by a respective branch of filters in a layer, one branch per domain. 
Each branch is learned only from domain-specific data, while non-branched layers are learned from data from all domains.
We then propose to extend basic branching to atom-decomposed branching as in Figure~\ref{fig:a2}, where domain characteristics are modeled by domain-specific dictionary atoms, and shared decomposition coefficients are enforced to align  cross-domain semantics.

\subsection{Pedagogical Branching Formulation}
We start with the simple-minded branching formulation in Figure~\ref{fig:a1}.
To model the domain-specific characteristics, at the first several convolutional layers, we dedicate a separate branch to each domain.
Domain shifts are modeled by an independent set of convolutional filters in the branch, 
trained respectively with errors propagated back from the loss functions of source and target domains.
For supervised learning, the loss function is the cross-entropy for each domain. For unsupervised learning, the loss function for the target domain can be either the feature statistics loss or the adversarial loss.
The remaining layers are shared across domains. 
We assume one target domain and one source domain in our discussion, while multiple domains are supported. 
Note that, though we adopt the source vs. target naming convention in the domain adaptation literature, we address here a general domain-invariant learning problem.

Domain-adaptive branching is simple and straightforward, however, it has the following drawbacks: 
First, both the number of model parameters and computation are multiplied with the number of domains.
Second, with limited target domain training data,
we can experience overfitting in determining a large amount of parameters dedicated to that domain.
Third,  no constraints are enforced to encourage cross-domain shared semantics.
We address these issues through layer branching with the proposed domain-adaptive filter decomposition.

\subsection{Atom-decomposed Branching}
To simultaneously counter domain shifts and enforce cross-domain shared semantics,  we decompose each convolutional filter in a branched layer into domain-specific dictionary atoms, and cross-domain shared coefficients, as illustrated in Figure~\ref{fig:dcf}.

\begin{wrapfigure}[12]{r}{0.56\textwidth}
		\centering
		\vspace{-4mm}
		\includegraphics[width=\linewidth]{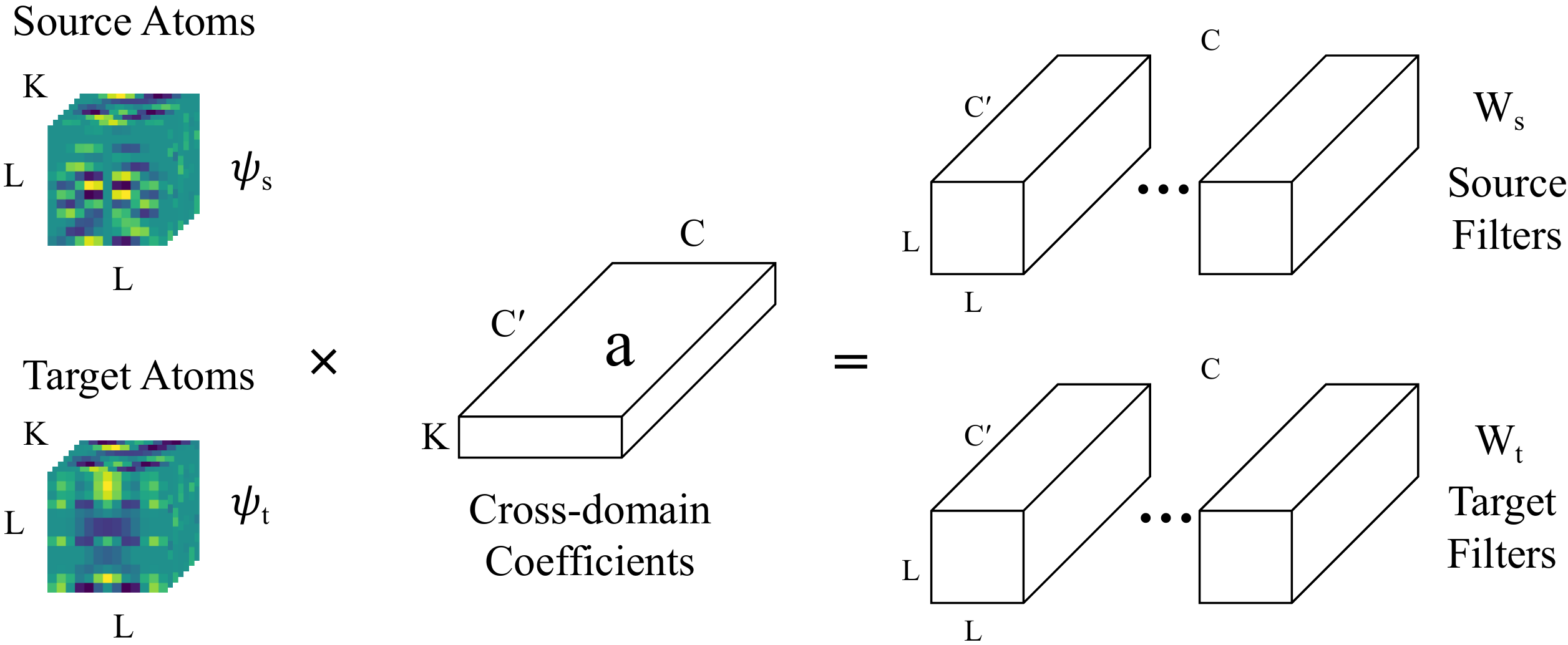}
		\caption{The proposed domain-adaptive filter decomposition for domain-invariant learning. }
		\label{fig:dcf}
\end{wrapfigure}

In our approach, we decompose source and target domain filters over domain-adaptive dictionary atoms, with decomposition coefficients shared across domains.
Specifically, at each branched layer, the source domain filter $W_s$ and target domain filter $W_t$ of size $L \times L \times C^\prime \times C $,
are decomposed as $W_s = \psi_s  a$ and $W_t = \psi_t a$, where $\psi_s$ and $\psi_t$ with a size of $L \times L \times K$ are $K$ domain-adaptive dictionary atoms for source and target domains, respectively; and $a \in \mathbb{R}^{K \times C^\prime \times C}$ denotes the common decomposed coefficients shared across domains.
Contrary to single domain works such as \cite{qiu2018dcfnet,sulam2018multilayer} that incorporate dictionaries into CNNs,
the domain-adaptive dictionary atoms are independently learned from the corresponding domain data to model domain shifts, and the shared decomposed coefficients are learned from the massive data from multiple domains.
Note that thanks to this proposed structure, only a small amount of additional parameters is required here to model each additional domain,
typically a few hundred.

With the above domain-adaptive filter decomposition,
at each branched layer, a regular convolution is now decomposed into two: First, a domain-specific atom convolves each individual input channel for domain shift  ``correction.'' Second, the ``corrected" output channels are weighted summed using domain-shared decomposition coefficients (1$\times$1 convolution) to promote common semantics. 
A toy example is presented in supplementary material Figure A.1 for illustrating the intuition behind the reason why manipulating dictionary atoms alone can address domain shifts.
We generate target domain data by applying two trivial operations to source domain images: First, every $3 \times 3$ non-overlapping patch in each image is locally rotated by $90^o$. Then, images are negated by multiplying with $-1$. 
Domain-invariant features are observed by manipulating dictionary atoms alone.
We will rigorously prove in Section~\ref{analysis} why such layer-wise ``correction'' aligns features across domains,
and present real-world examples in the experiments.

\paragraph{Parameters and Computation Reduction.}
Suppose that both input and output features have the same spatial resolution of $W \times W$, in each forward pass in a regular convolutional layer, there are totally $W^2 \times C^\prime \times C \times (2L^2  + 1)$ flops for each domain.
While in our model, each domain only introduces $W^2 \times C^\prime \times 2K (L^2 + C)$ flops, where $K$ is the number of dictionary atoms.
For parameters,
there are totally $D \times C^\prime \times C \times L^2$ parameters in a regular convolutional layer where $D$ is the number of domains which is typically 2 in our case. In our model, each layer has only $K \times (C^\prime \times C + D \times L^2)$ parameters.
Taking VGG-16 \cite{simonyan2014very} as an example with an input size of $224\times224$, a regular VGG-16 with branching, Fig~\ref{fig:a1} and \cite{rozantsev2018beyond,rozantsev2018residual}, requires adding 14.71M parameters and 15.38G flops in convolutional layers to handle each additional domain. With the proposed method (Fig~\ref{fig:a2}), VGG-16 only requires adding 702 parameters and 10.75G flops to handle one additional domain (K=6).

\section{Provable Invariance with Adaptive Atoms}
\label{analysis}

In this section, we theoretically prove that 
the features produced by the source and target networks from domain-transferred inputs
can be aligned by the proposed framework of only adjusting multi-layer atoms,
assuming a generative model of the source and target domain images via CNN.
Since convolutional generative networks are a rich class of models for domain transfer 
\cite{hu2018duplex,murez2017image},
our analysis provides a theoretical justification of the proposed approach, providing a contribution to the theoretical foundations of domain adaptation
. 
Diverse examples in the experiment section show the applicability of the proposed approach 
is potentially larger than what is proved here. All proofs are in the supplementary material Section D.

\paragraph{Filter Transform via Atom Transform.}
Let $w_s$ and $w_t$ be the filters in the branched convolutional layer for 
the source and target domains respectively,
and similarly denote the source and target atoms by 
$\psi_{k,s}$ and $\psi_{k,t}$.
In the proposed atom decomposition architecture,
the source and target domain filters are
linear combinations of the 
domain-specific dictionary atoms with shared decomposition coefficients,
namely 
\[
w_s(u) = \sum_k a_k \psi_{k,s}(u), 
\quad 
w_t(u) = \sum_k a_k \psi_{k,t}(u).
\]
Certain transforms of the filter can be implemented by only transforming the dictionary atoms,
including 

\begin{itemize}
	
	\item[(1)] A linear correspondence of filter values.
	Let $\lambda: \R \to \R$ be a linear mapping, by linearity,
	\begin{align*}
	 ~~~~ \psi_{k,s}(u) \to \psi_{k,t}(u) = \lambda( \psi_{k,s}(u)) 
	 \text{ applies } w_s(u) \to w_t(u) = \lambda( w_s(u)).
	\end{align*}
	E.g. the negation $\lambda(\xi) = - \lambda(\xi)$, 
	as shown in supplementary material Figure A.1. 
	
	\item[(2)] The transformation induced by a displacement of spatial variable, i.e., ``spatial transform'' of filters,
	defined as 
	$D_\tau w(u) = w( u - \tau(u))$,
	where $\tau: \R^2 \to \R^2$ is a differentiable displacement field.
	Note that the dependence on spatial variable $u$ in a filter is via the atoms,
	thus
	$\psi_{k,s} \to \psi_{k,t} = D_\tau \psi_{k,s}
	\text{  applies  }
	w_s \to w_t = D_\tau w_s.$
\end{itemize}

If such filter adaptations are desired in the branching network,
then it suffices to branch the dictionary atoms while keeping the coefficients $a_k$
shared, as implemented in the proposed architecture shown in Figure~\ref{fig:main}(c). 
A fundamental question is thus 
how large is the class of possible domain shifts that can be corrected 
by these ``allowable'' filter transforms.
In the rest of the section,
we show that if the domain shifts in the images 
are induced from a generative CNN 
where the filters for source and target differ by a sequence of allowable transforms,
then the domain shift can be provably eliminated by another sequence of filter transforms
which can be implemented by atom branching only.

\paragraph{Provable Invariance.}
Stacking the approximate commuting relation, \textit{Lemma 1} in supplementary material Section D, in multiple layers
allows to correct a sequence of filter transforms in previous convolutional layers
by another sequence of ``symmetric'' ones. 
This means that
if we impose a convolutional generative model  on the source and target input images,
and assume that 
the domain transfer results from a sequence of 
spatial transforms of filters in the generative net,
then by correcting these filter transforms in the subsequent convolutional layers
we can guarantee the recovery of the same feature mapping.
The key observation is that the filter transfers can be implemented by atoms transfer only.

We summarize the standard theoretical assumptions as follows:

\begin{itemize}
	\item[(A1)] The nonlinear activation function $\sigma$ in any layer is non-expansive,
	
	\item[(A2)] In the generative net (where layer is indexed by negative integers),   
	$w_t^{(-l)} = D_l w_s^{(-l)}$,
	where $D_l = D_{\tau_l}$, $\tau_l$ is odd and $| \nabla \tau_l |_\infty \le \varepsilon < \frac{1}{5}$
	for all $l=1,\cdots, L$.
	The biases in the target generative net are mildly adjusted accordingly 
	due to technical reasons
	(to preserve the ``baseline output'' from zero-input, c.f. detail in the proof).
	
	\item[(A3)]
	In the generative net,
	$\|w_s^{(-l)}\|_1 \le 1$ for all $l$,
	and so is $w_t^{(-l)} = D_l w_s^{(-l)}$. 
	Same for the feed-forward convolutional net taking the generated images as input, called ``feature net'':
	The source net filters have $\|w_s^{(l)}\|_1 \le 1$ for $l=1,2\cdots$,
	and same with $D_l w_s^{(l)} $ which will be set to be $ w_t^{(l)} $.
	Also, $w_s^{(-l)}$ and $w_s^{(l)}$ are both supported on $2^{j_l}B$ for $l=1,\cdots, L$.
\end{itemize}

One can show that $\| D_\tau w \|_1 =  \|w\|_1 $ when $(I_d-\rho)$ is a rigid motion,
and generally $ | \| D_\tau w \|_1 - \| w \|_1 | \le c |\nabla \tau|_\infty \|w\|_1$ which is negligible when $\varepsilon$ is small.
Thus in (A3) the boundedness of the 1-norm of the source and target filters imply one another exactly or approximately. 
The boundedness of 1-norm of the filters preserves the non-expansiveness of the mapping from input to output in 
a convolutional layer, and in practice is qualitatively preserved by normalization layers. 
Also, as a typical setting,
(A3) assumes that the scales $j_l$ in the generative net (the ($-l$)-th layer) 
and the feature net (the $l$-th layer) 
are matched, which simplifies the analysis and can be relaxed.

\begin{theorem}\label{thm:main}
	Suppose that $X_s$ and $X_t$ are source and target images 
	generated by $L$-layer generative CNN nets 
	with source and target filters $w_s^{(-l)}$, $w_t^{(-l)}$  respectively
	from the common representation $h$.
	Under (A1)-(A3),
	the output at the $L$-th layer of the target feature CNN from $X_t$, 
	by setting $w_t^{(l)} = D_l w_s^{(l)}$ in all layers which can be implemented by atom branching,
	approximates that of the source feature CNN from $X_s$
	up to an error which is bounded in 1-norm by
$
	4 \varepsilon \left\{
	(\sum_{l=1}^L 2^{j_l})  \| \nabla h \|_1  + 2 L \|h\|_1
	\right\},
$
	and the second term vanishes if $(I_d - \tau_l)$ are rigid motions,  e.g., rotation.
\end{theorem}

\FloatBarrier
\section{Experiments}
\label{exp}

In this section, we perform extensive experiments to evaluate the performance of the proposed domain-adaptive filter decomposition.
We start with the comparisons among the 3 architectures listed in Figure~\ref{fig:main} on two supervised tasks.
To demonstrate the proposed framework as one principled way for domain-invariant learning, we then conduct a set of domain adaptation experiments.
There we show, by simply plugging the proposed domain filter decomposition into regular CNNs used in existing domain adaptation methods, we consistently observe performance improvements, 
which well-illustrate that our method is orthogonal to other domain adaptation methods.

\subsection{Architecture Comparisons}
We start with two supervised tasks performed on the three architectures listed in Figure~\ref{fig:main}, regular CNN (A1), basic branching (A2), and branching with domain-adaptive filter decomposition (A3).
The networks with DAFD are trained end-to-end with a summed loss for domains, and the domain-specific atoms are only updated by the error from the corresponding domain, while the decomposition coefficients are updated by the joint error across domains.
\begin{figure}
	\resizebox{\textwidth}{!}{%
		\begin{minipage}{0.45\linewidth}
			\captionof{table}{Accuracy (\%) on MNIST->SVHN for supervised domain adaptation. A1, A2, and A3 correspond to regular CNN, basic branching, and branching with DAFD shown in Figure~\ref{fig:main}, respectively.}
			\label{tab:exp0}
			\resizebox{\textwidth}{!}{
			\begin{tabular}{c|c c c | c c c}
				\toprule
				\multirow{2}*{Scales} & \multicolumn{3}{c|}{Source domain} &\multicolumn{3}{c}{Target domain}\\
				~ & 0.1 & 0.05  & 0.005 & 0.1 & 0.05 & 0.005 \\
				\midrule
				A1 & 98.4 & 96.4  & 98.0 & 81.6 & 80.2  & 61.0\\
				A2 &  99.2 & 98.6  & 97.6 & 81.4 & 78.4  & 49.6 \\
				\textbf{A3} & \textbf{99.4 }& \textbf{98.8}  & \textbf{98.8} & \textbf{85.6} & \textbf{82.2} & \textbf{64.4} \\
				\bottomrule
			\end{tabular}}
		\end{minipage}
		\hspace{3mm}
		\begin{minipage}{0.45\linewidth}
			\captionof{table}{Cross-domain simultaneous face recognition on NIR-VIS-2.0. A1, A2, and A3 correspond to regular CNN, basic branching, and branching with domain-adaptive filter decomposition shown in Figure~\ref{fig:main}, respectively.}
			\label{t0}
			\resizebox{\textwidth}{!}{
			\begin{tabular}{r|c c c}
				\toprule
				Methods	& VIS Acc (\%) & NIR Acc (\%) & NIR+VIS (\%)\\
				\midrule
				A1 & 75.57 & 52.71   & 98.44 \\
				A2 &  94.46 & 87.50 & 98.58 \\
				\textbf{A3} &\textbf{97.16} & \textbf{95.03} & \textbf{99.15}  \\
				\bottomrule
			\end{tabular}}
		\end{minipage}
	}
\end{figure}

\paragraph{Supervised domain adaptation on images.}
The first task is supervised domain adaptation, where we adopt a challenging setting by using MNIST as the source domain, and SVHN as the target domain. We perform a series of experiments by progressively reducing the annotated training data for the target domain. 
We start the comparisons at 10\% of the target domain labeled samples, and end at 0.5\% where only 366 labeled samples are available for the target domain.
The results on test set for both domains are presented in Table~\ref{tab:exp0}.
It is clearly shown that when training the target domain with small amount of data, a network with basic branching suffers from overfitting to the target domain because of the large amount of domain specific parameters. While regular CNN generates well on target domain, the performance on source domain degrades when the number of target domain data is comparable. 
A network with the proposed domain-adaptive filter decomposition significantly balances the learning of both the source and the target domain, and achieves best accuracies on both domains regardless of the amount of annotated training data for the target domain.
The feature space of the three candidate architectures are visualized in Figure~\ref{fig:space}.

\begin{figure*}[h]
	\centering

	\subfigure[Feature space in (a)]{
		\includegraphics[width=0.302\linewidth]{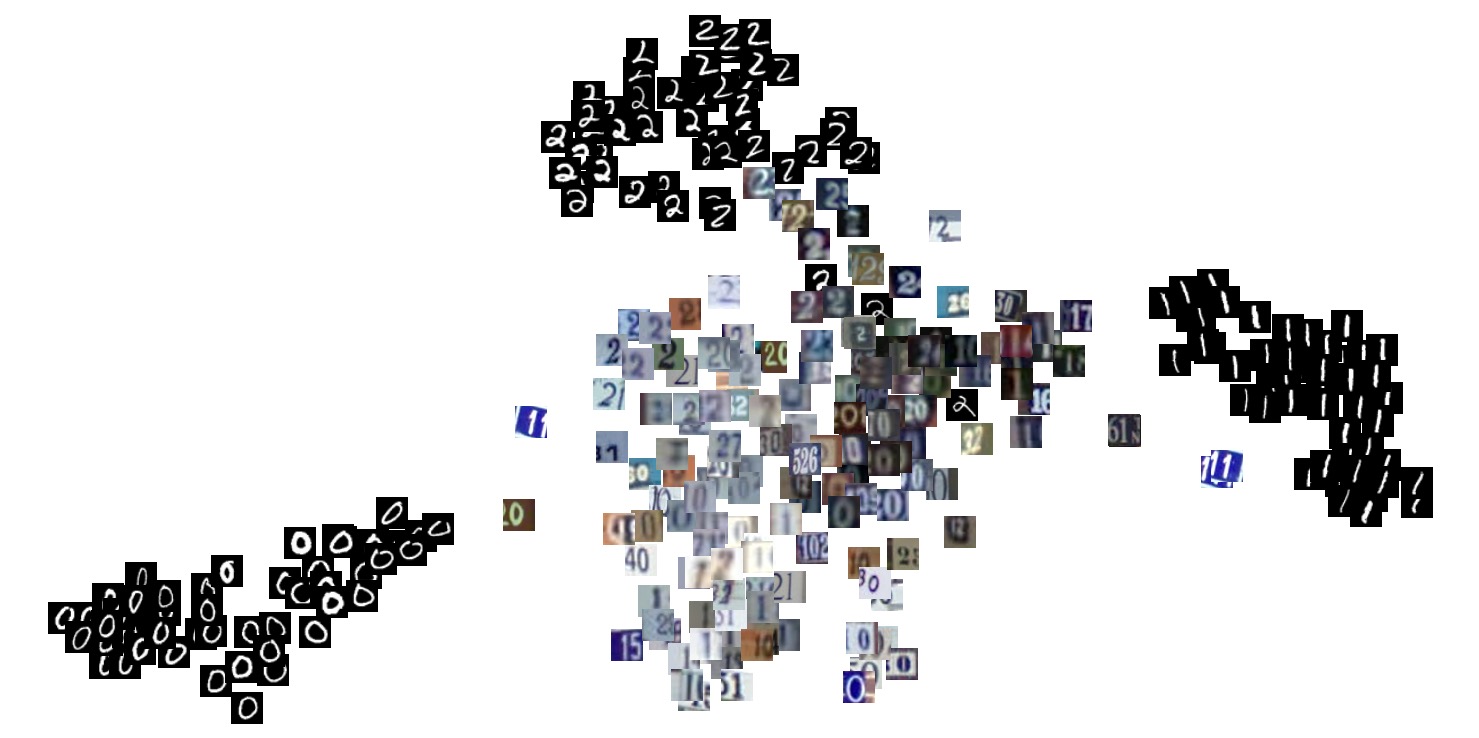}
		\label{fig:p0}
	}
	\subfigure[Feature space in (b)]{
		\includegraphics[width=0.302\linewidth]{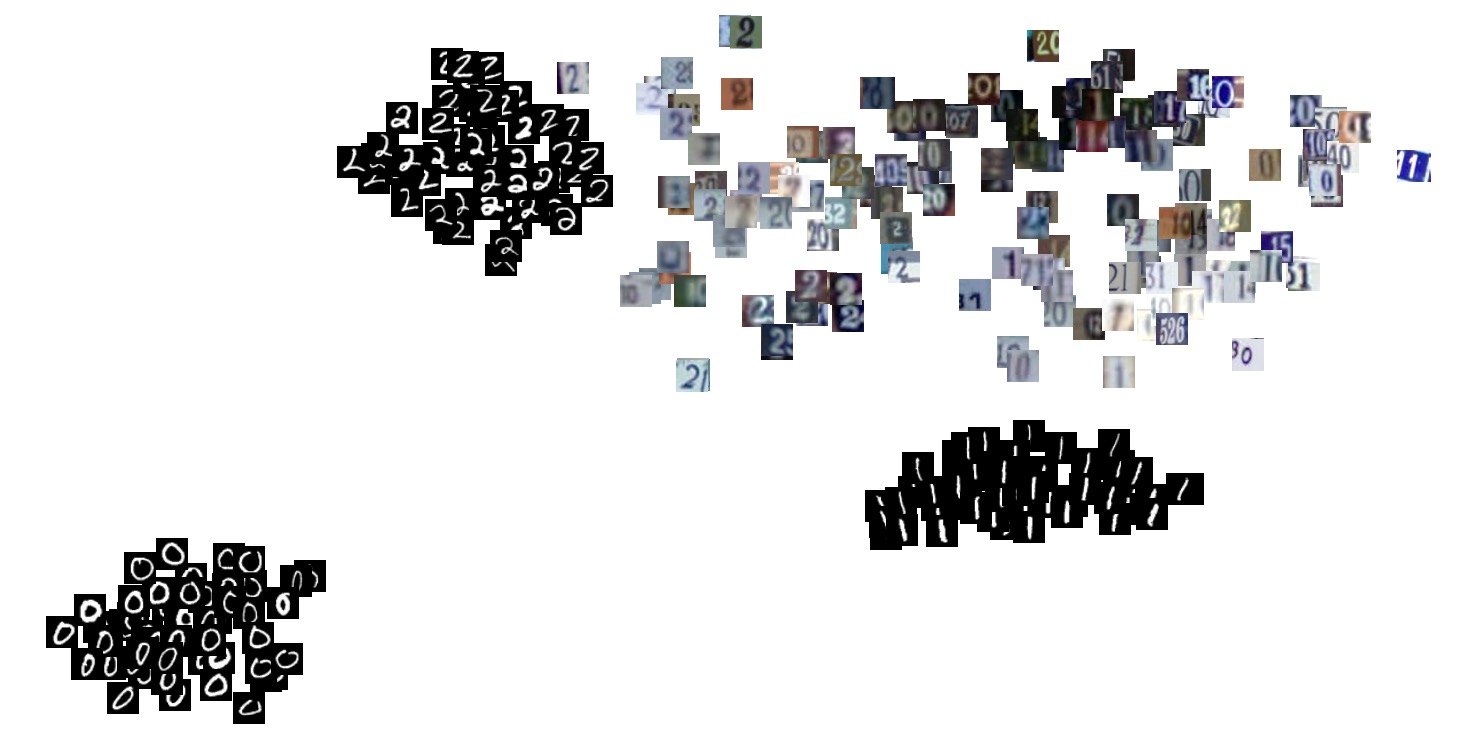}
		\label{fig:p1}
	}
	\subfigure[Feature space in (c)]{
		\includegraphics[width=0.302\linewidth]{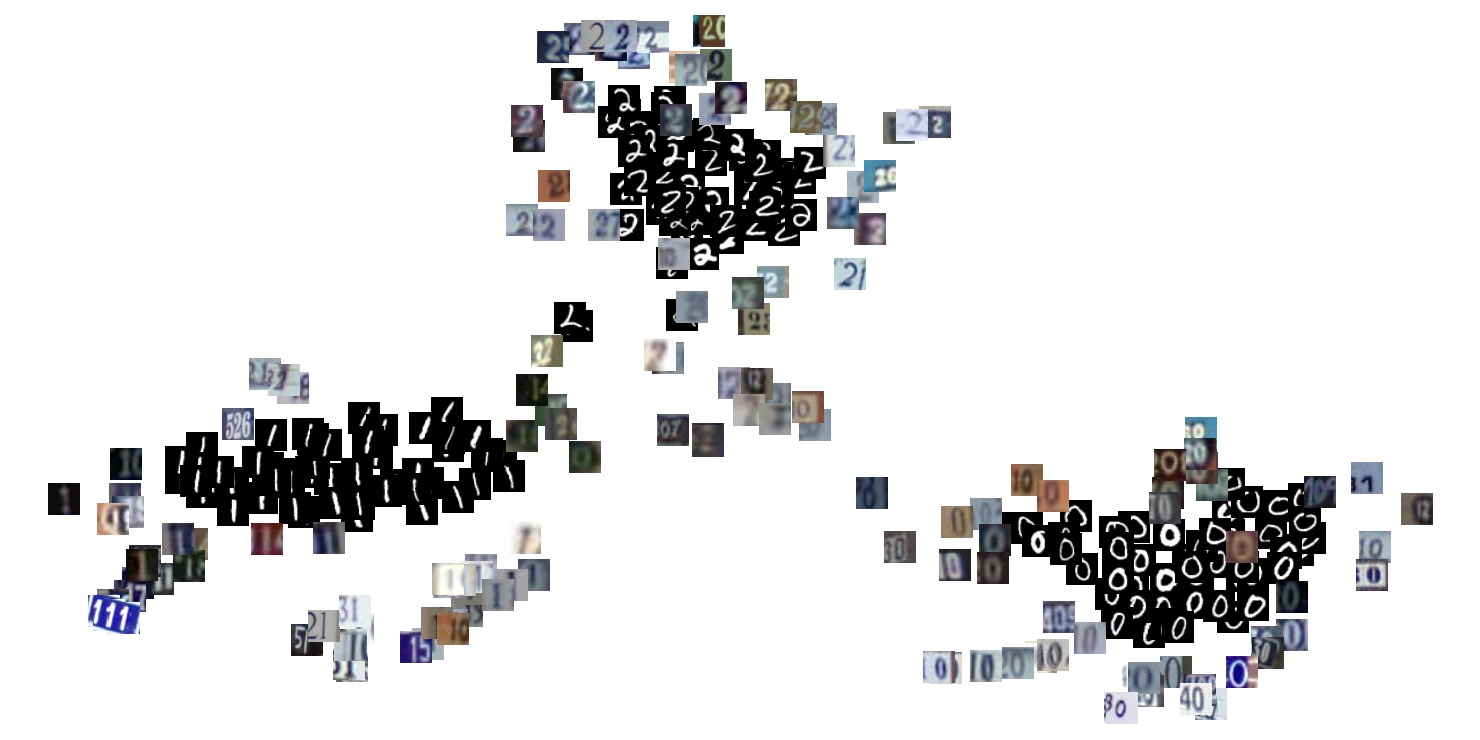}
		\label{fig:p2}
	}
	\caption{ The feature space of the three candidate architectures in Figure~\ref{fig:main}, MNIST $\rightarrow$ SVHN, are visualized using t-SNE \cite{maaten2008visualizing} in (a), (b), (c), respectively.
		The obtained superior cross-domain invariance of the proposed framework can be clearly observed in (c).
	}
	\label{fig:space}
\end{figure*}

\paragraph{Supervised simultaneous cross-domain face recognition.}
Besides standard domain adaptation, the proposed domain-adaptive filter decomposition can be extended to general tasks that involves more than one visual domain; domain adaptation is performed without loosing the power of the original domain and multiple-domains can be simultaneously exploited. Here we demonstrate this by performing experiments on supervised cross-domain face recognition.
We adopt the NIR-VIS 2.0 \cite{li2013casia}, which consists of 17,580 NIR (near infrared) and VIS (visible light) face images of 725 subjects, and perform cross-domain face recognition.
We adopt VGG16 as the base network structure, branch all the convolutional layers with the proposed domain-adaptive filter decomposition, and train the network from scratch. In each convolutional layer, two set of dictionary atoms are trained for modeling the NIR and the VIS domain, respectively.
Specifically, one VIS image and one NIR image are fed \textit{simultaneously} to the network, and the feature vectors of both domains are averaged to produce the final cross-domain feature, which is further fed into a linear classifier for classifying the identity.
While the training is conducted using both domains simultaneously, we test the network under three settings including feeding single domain inputs only (VIS Acc and NIR Acc in Table~\ref{t0}) and both domain inputs (VIS+NIR Acc in Table~\ref{t0}).
Quantitative comparisons demonstrate that branching with the proposed DAFD performs superiorly even with a missing input domain.
Note that A2 requires additional 14.71M parameters over A1, while our method requires only 0.0007M as shown in supplementary material Table A.1.

\subsection{Experiments on Standard Domain Adaptations}
In this section, we perform extensive experiments on unsupervised domain adaptation.
Note that the objective of the experiments in this section is not to validate the proposed domain-adaptive filter decomposition as just another new method for domain adaptation. Instead, since most of the state-of-the-art domain adaptation methods adopt the regular CNN (A1) with completely shared parameters for domains, we show the compatibility and the generality of the proposed domain-adapting filter decomposition by plugging it into underlying domain adaptation methods, and evaluate the effectiveness by retraining the networks using exactly the same setting and observing the performance improvement over the underlying methods.
Diverse real-world domain shifts including different sensors, different image sources, and synthetic images, and applications on both classification and segmentation are examined in these experiments.
Together with the experiments in the previous section, this further stresses the plug-and-play virtue of the proposed framework.

\begin{wraptable}[9]{r}{0.5\textwidth}
	\centering 
	\vspace{-3mm}
	\caption{Accuracy (\%) on {Digits} for unsupervised domain adaptation.}
	\label{tab:digit}
	\resizebox{0.5\textwidth}{!}{%
		\begin{tabular}{c|cccc}
			\toprule
			Methods & M $\rightarrow$ U & U $\rightarrow$ M & S $\rightarrow$ M & Avg. \\
			\midrule
			DANN & - & - & 73.9 \\
			ADDA & 89.4 & 90.1 & 76.0 & 85.1 \\
			CDAN+E  & 95.6 & 98.0 & 89.2 & 94.3  \\
			\midrule
			DANN + DAFD   & 92.0 & 95.2 & 82.1 (\textbf{11.1}$\%\uparrow$) & 89.8 \\
			ADDA + DAFD   &91.4 & 94.8 & 82.9 & 89.7 (\textbf{5.5}$\%\uparrow$)\\
			CDAN+E + DAFD &  96.8 & 98.8 & 96.6 & 97.4 (\textbf{3.2}$\%\uparrow$)\\
			
			\bottomrule
			
		\end{tabular}
		
	}	
\end{wraptable}

In practise, instead of learning independent source and target domain atoms, we learn the {\it residual} between the source and the target domain atoms. The residual is initialized by full zeros, and trained by loss for encouraging invariant features in the underlying methods, e.g., the adversarial loss in ADDA \cite{tzeng2017adversarial}. We consistently observe that this stabilizes the training and promotes faster convergence.

\paragraph{Image classification.}
We perform experiments on three public digits datasets: MNIST, USPS, and Street View House Numbers (SVHN), with three transfer tasks: USPS to MNIST (U
$\rightarrow$ M), MNIST to USPS (M $\rightarrow$ U), and SVHN to MNIST (S $\rightarrow$ M). Classification accuracy on the target domain test set samples is adopted as the metric for measuring the performance. 
We perform domain-adaptive domain decomposition on state-of-the-art methods DANN \cite{ganin2016domain}, ADDA \cite{tzeng2017adversarial}, and CDAN+E \cite{long2018conditional}. Quantitative comparisons are presented in Table~\ref{tab:digit}, demonstrating significant improvements over underlying methods.

\paragraph{Office-31.}
Office-31 \cite{saenko2010adapting} is one of the most widely used datasets for visual domain adaptation, which has 4,652 images and 31 categories collected from three distinct domains: Amazon (\textbf{A}), Webcam (\textbf{W}), and DSLR (\textbf{D}). We evaluate all methods on six transfer tasks \textbf{A} $\rightarrow$ \textbf{W}, \textbf{D} $\rightarrow$ \textbf{W}, \textbf{W} $\rightarrow$ \textbf{D}, \textbf{A} $\rightarrow$ \textbf{D}, \textbf{D} $\rightarrow$ \textbf{A}, and \textbf{W} $\rightarrow$ \textbf{A}. 
Two feature extractors, AlexNet \cite{krizhevsky2012imagenet} and ResNet \cite{he2016deep} are adopted for fair comparisons with underlying methods.
Specifically, ImageNet initialization are widely used for ResNet in the experiments with Office-31, and we consistently observe that initialization is important for the training on Office-31.
Therefore, when training ResNet based networks with domain-adaptive filter decomposition, we initialize the feature extractor using parameters decomposed from ImageNet initialization. 
The quantitative comparisons are in Table~\ref{tab:office31}.

\begin{table*}[h]
	\addtolength{\tabcolsep}{2pt}
	\centering
	\caption{Accuracy (\%) on {Office-31} for unsupervised domain adaptation (AlexNet and ResNet).}
	\label{tab:office31}
	\resizebox{\textwidth}{!}{%
		\scriptsize
		\begin{tabular}{c|c|ccccccc}
			\toprule
			~&Method & A $\rightarrow$ W & D $\rightarrow$ W & W $\rightarrow$ D & A $\rightarrow$ D & D $\rightarrow$ A & W $\rightarrow$ A & Avg. \\
			\midrule
			\multirow{5}*{\rotatebox{90}{AlexNet}} & AlexNet (no adaptation) & 61.6$\pm$0.5 & 95.4$\pm$0.3 & 99.0$\pm$0.2 & 63.8$\pm$0.5 & 51.1$\pm$0.6 & 49.8$\pm$0.4 & 70.1 \\
			~&DANN \cite{ganin2016domain} & 73.0$\pm$0.5 & 96.4$\pm$0.3 & 99.2$\pm$0.3 & 72.3$\pm$0.3 & 53.4$\pm$0.4 & 51.2$\pm$0.5 & 74.3 \\
			~&ADDA \cite{tzeng2017adversarial} & 73.5$\pm$0.6 & 96.2$\pm$0.4 & 98.8$\pm$0.4 & 71.6$\pm$0.4 & 54.6$\pm$0.5 & 53.5$\pm$0.6 & 74.7 \\
			\cmidrule{2-9}
			~&DANN + DAFD & 74.4$\pm$0.3 & 97.1$\pm$0.4 & 99.1$\pm$0.4 & 74.2$\pm$0.3 & 56.8$\pm$0.5 & 53.1$\pm$0.7 & 75.8 (\textbf{2.3}$\%\uparrow$) \\
			~&ADDA + DAFD & 77.2$\pm$0.5 & 97.9$\pm$0.4 & 98.5$\pm$0.2 & 73.2$\pm$0.4 & 55.4$\pm$0.6 & 57.8$\pm$0.5 & 76.7 (\textbf{2.7}$\%\uparrow$) \\
			\midrule
			\midrule
			\multirow{7}*{\rotatebox{90}{ResNet}} & ResNet-50 (no adaptation) & 68.4$\pm$0.2 & 96.7$\pm$0.1 & 99.3$\pm$0.1 & 68.9$\pm$0.2 & 62.5$\pm$0.3 & 60.7$\pm$0.3 & 76.1 \\	
			~&DANN \cite{ganin2016domain} & 82.0$\pm$0.4 & 96.9$\pm$0.2 & 99.1$\pm$0.1 & 79.7$\pm$0.4 & 68.2$\pm$0.4 & 67.4$\pm$0.5 & 82.2 \\
			~&ADDA \cite{tzeng2017adversarial} & 86.2$\pm$0.5 & 96.2$\pm$0.3 & 98.4$\pm$0.3 & 77.8$\pm$0.3 & 69.5$\pm$0.4 & 68.9$\pm$0.5 & 82.9 \\ 
			~&CDAN+E \cite{long2018conditional} & 94.1$\pm$0.1 & 98.6$\pm$0.1 & 100.0$\pm$.0 & 92.9$\pm$0.2 & 71.0$\pm$0.3 & {69.3}$\pm$0.3 & 87.7 \\
			\cmidrule{2-9}
			~&DANN + DAFD & 86.4$\pm$0.4 & 96.8$\pm$0.2 & 99.2$\pm$0.1 & 84.4$\pm$0.4 & 70.5$\pm$0.4 & 68.8$\pm$0.4 & 84.35 (\textbf{2.3}$\%\uparrow$) \\
			~&ADDA + DAFD & 86.8$\pm$0.4 & 97.7$\pm$0.1 & 98.4$\pm$0.1 & 80.5$\pm$0.3 & 71.1$\pm$0.4 & 69.1$\pm$0.5 & 83.9 (\textbf{1.2}$\%\uparrow$)\\
			~&CDAN+E + DAFD & 95.6$\pm$0.1 & 98.8$\pm$0.1 & 100.0$\pm$0.0 & 93.5$\pm$0.2 & 76.6$\pm$0.5 & 71.3$\pm$0.4 & 89.3 (\textbf{1.8}$\%\uparrow$)\\
			\bottomrule
		\end{tabular}
	}
\end{table*}

\paragraph{Image segmentation.}

Beyond image classification tasks, we perform a challenging experiment on image segmentation to demonstrate the generality of the proposed domain-adaptive filter decomposition.
We perform unsupervised adaptation from the GTA dataset \cite{richter2016playing} (images generated from video games) to the Cityscapes dataset \cite{cordts2016cityscapes} (real-world images), which has a significant practical value considering the expensive cost on collecting annotations for image segmentation in real-world scenarios. 
Two underlying methods FCNs in the wild \cite{hoffman2016fcns} and AdaptSegNet \cite{tsai2018learning} are adopted for comprehensive comparisons.
Based on the underlying methods, all the convolutional layers are decomposed using domain-adaptive filter decomposition, and all the transpose-convolutional layers are kept sharing by both domains.
For quantitative results in Table~\ref{tab:ss}, we use intersection-over-union, i.e., IoU = $\rm \frac{TP}{TP+FP+FN}$, where TP, FP, and FN are the numbers of true positive, false positive, and false negative pixels, respectively, as the evaluation metric.
As with the previous examples, our method improves all state-of-the-art architectures.
Qualitative results are shown in Figure~\ref{fig:seg} and supplementary material Figure A.2, and data samples are in Figure~\ref{fig:gta}.

\begin{table*}[h!]
	\begin{center} 
		\caption{Unsupervised DA for semantic segmentation: GTA $\rightarrow$ Cityscapes}
		\label{tab:ss}
		\centering
		\resizebox{\textwidth}{!}{%
			\begin{tabular}{c|c|c|c|c|c|c|c|c|c|c|c|c|c|c|c|c|c|c|c|c}
				\toprule
				\multirow{2}*{Methods}& \multirow{2}*{IoU} & \multicolumn{19}{c}{Class-wide IoU}\\
				\cline{3-21} 
				~&~&\rotatebox{45}{road}&\rotatebox{45}{sidewalk}&\rotatebox{45}{building}&\rotatebox{45}{wall}&\rotatebox{45}{fence}&\rotatebox{45}{pole}&\rotatebox{45}{t-light}&\rotatebox{45}{t-sign}&\rotatebox{45}{veg}&\rotatebox{45}{terrain}&\rotatebox{45}{sky}&\rotatebox{45}{person}&\rotatebox{45}{rider}&\rotatebox{45}{car}&\rotatebox{45}{truck}&\rotatebox{45}{bus}&\rotatebox{45}{train}&\rotatebox{45}{mbike}&\rotatebox{45}{bicycle}\\
				\midrule
				No Adapt (VGG) &17.9& 26.0& 14.9 &65.1& 5.5& 12.9 &8.9& 6.0& 2.5& 70.0 &2.9& 47.0 &24.5& 0.0 &40.0 &12.1& 1.5& 0.0& 0.0& 0.0\\
				No Adapt (ResNet) & 36.6 & 75.8 & 16.8 & 77.2 & 12.5 & 21.0 & 25.5 & 30.1 & 20.1 & 81.3 & 24.6 & 70.3 & 53.8 & 26.4 & 49.9 & 17.2 & 25.9 & 6.5 & 25.3 & 36.0 \\
				FCN WLD (VGG) &27.1& 70.4& 32.4& 62.1& 14.9& 5.4& 10.9& 14.2& 2.7& 79.2& 21.3& 64.6 &44.1 &4.2& 70.4& 8.0 &7.3& 0.0& 3.5& 0.0  \\
				AdaptSegNet (VGG) &35.0 &87.3& 29.8& 78.6& 21.1 &18.2& 22.5 &21.5 &11.0& 79.7 &29.6 &71.3 &46.8& 6.5 &80.1 & 23.0  &26.9 &0.0&10.6& 0.3\\
				AdaptSegNet (ResNet) & 42.4 & 86.5&  36.0 & 79.9 & 23.4 & 23.3 & 23.9 & 35.2 & 14.8 & 83.4 & 33.3 & 75.6 & 58.5 & 27.6 & 73.7 & 32.5 & 35.4 & 3.9 & 30.1 & 28.1 \\
				\midrule
				FCN WLD + DAFD & 32.7 (\textbf{20.7}$\%\uparrow$) & 76.4 & 36.7 & 68.8 & 17.6 & 5.8 & 11.1 & 13.9 & 2.9 & 80.0 & 24.4 & 69.1 & 47.5 & 4.3 & 74.4 & 14.1 & 6.3 & 0.0 & 2.1 & 0.0\\
				AdaptSegNet (VGG) + DAFD& 36.4 (\textbf{4.0}$\%\uparrow$) & 86.7 & 35.3 & 78.8 & 22.8 & 14.5 & 23.9 &  21.9 & 18.2& 82.1 & 32.2 &  66.8 & 49.6 & 10.1 & 81.2 & 19.6 & 27.1 & 1.1 & 11.4 & 4.2\\
				AdaptSegNet (ResNet) + DAFD & 45.0 (\textbf{6.1}$\%\uparrow$) & 88.2 & 38.5 & 8.12 & 25.0 & 23.8 & 22.9 & 35.1 & 14.4 & 84.9 & 34.1 & 79.9 & 59.5 & 29.1 & 75.5 & 30.1 & 35.2& 2.9& 28.7 & 29.1  \\
				\bottomrule
		\end{tabular}}
		\vspace{-4mm}
	\end{center}
\end{table*}

\begin{figure}[h]
	\centering
	\subfigure[Target domain image.]{
		\begin{minipage}[b]{0.23\linewidth}
			\includegraphics[width=1\linewidth]{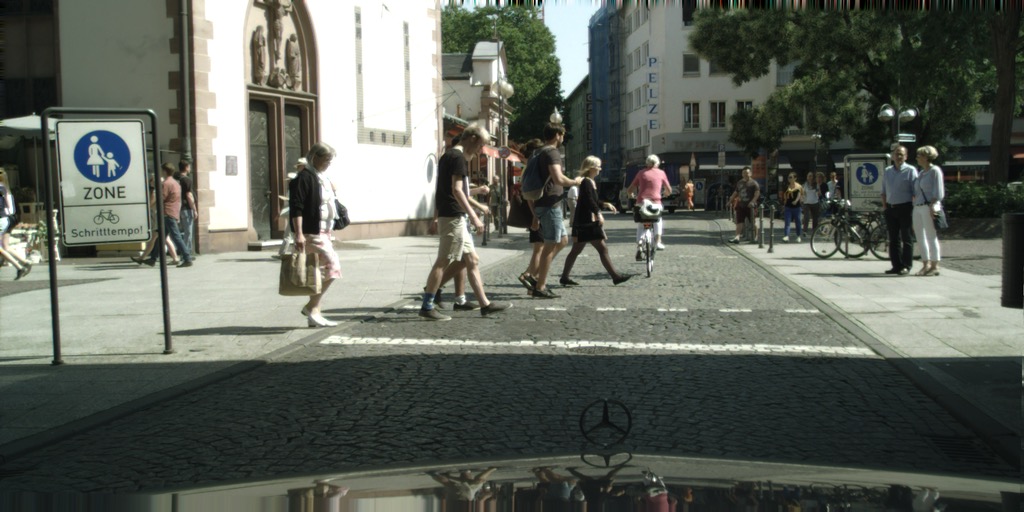}\vspace{3pt}\\
			\includegraphics[width=1\linewidth]{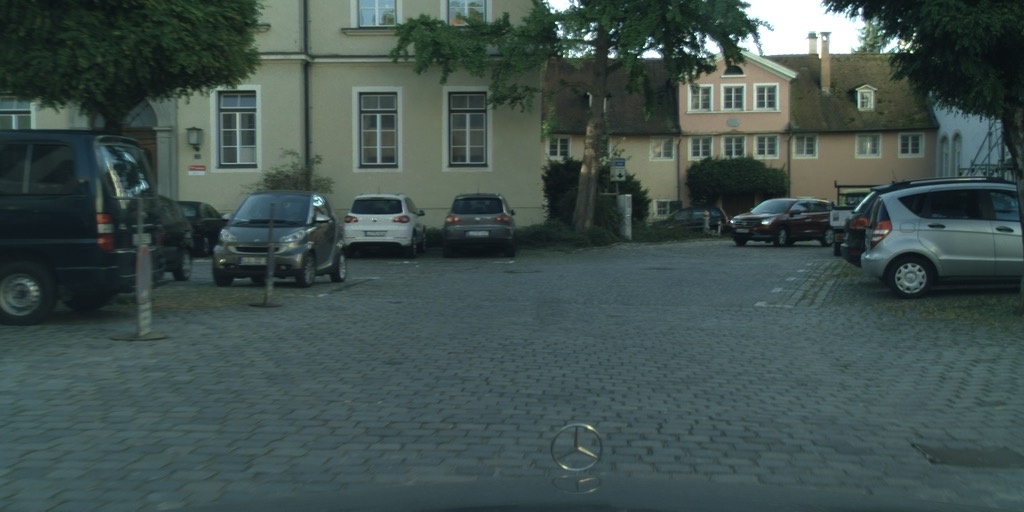}
	\end{minipage}}
	\subfigure[Before adaptation.]{
		\begin{minipage}[b]{0.23\linewidth}
			\includegraphics[width=1\linewidth]{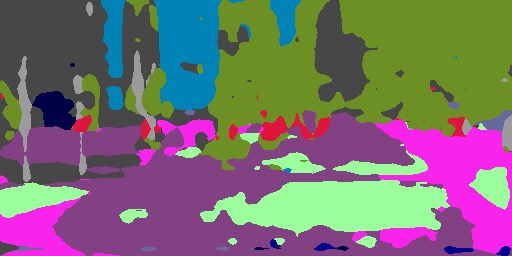}\vspace{3pt}\\
			\includegraphics[width=1\linewidth]{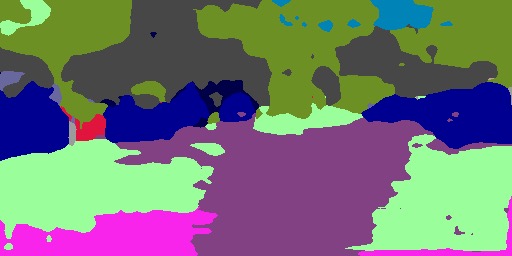}
	\end{minipage}}
	\subfigure[After adaptation.]{
		\begin{minipage}[b]{0.23\linewidth}
			\includegraphics[width=1\linewidth]{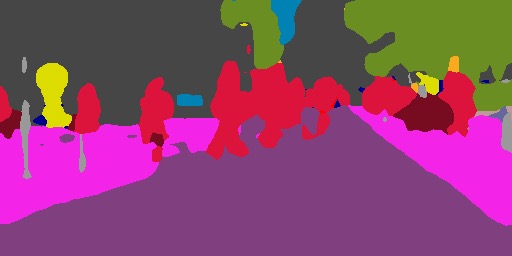}\vspace{3pt}\\
			\includegraphics[width=1\linewidth]{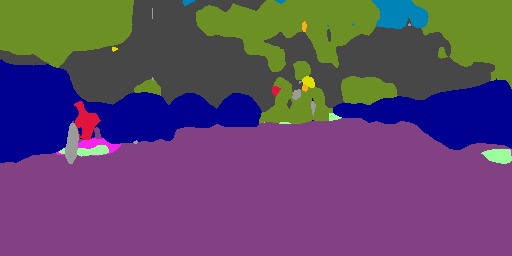}
	\end{minipage}}
	\subfigure[Ground truth.]{
		\begin{minipage}[b]{0.23\linewidth}
			\includegraphics[width=1\linewidth]{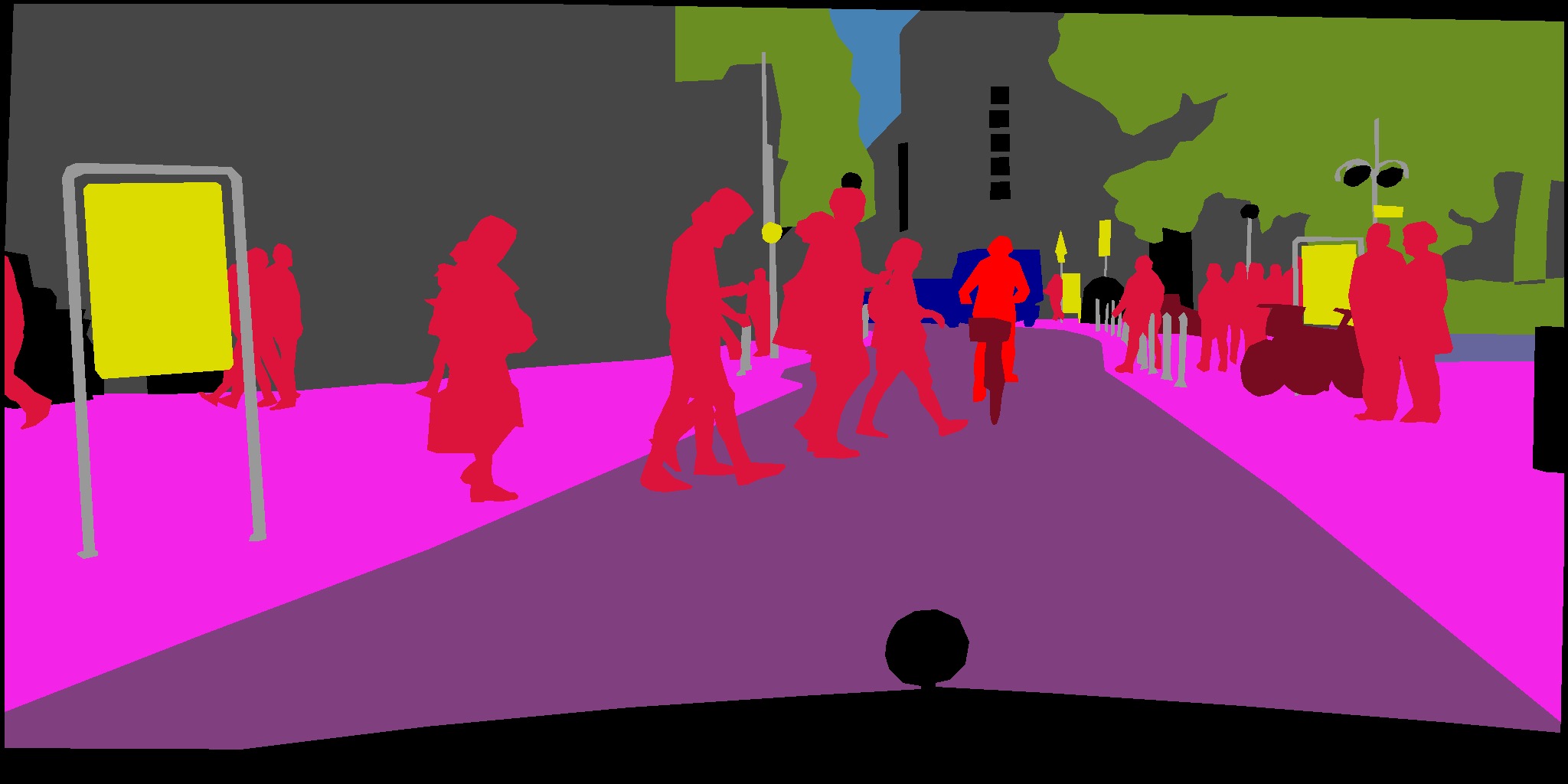}\vspace{3pt}\\
			\includegraphics[width=1\linewidth]{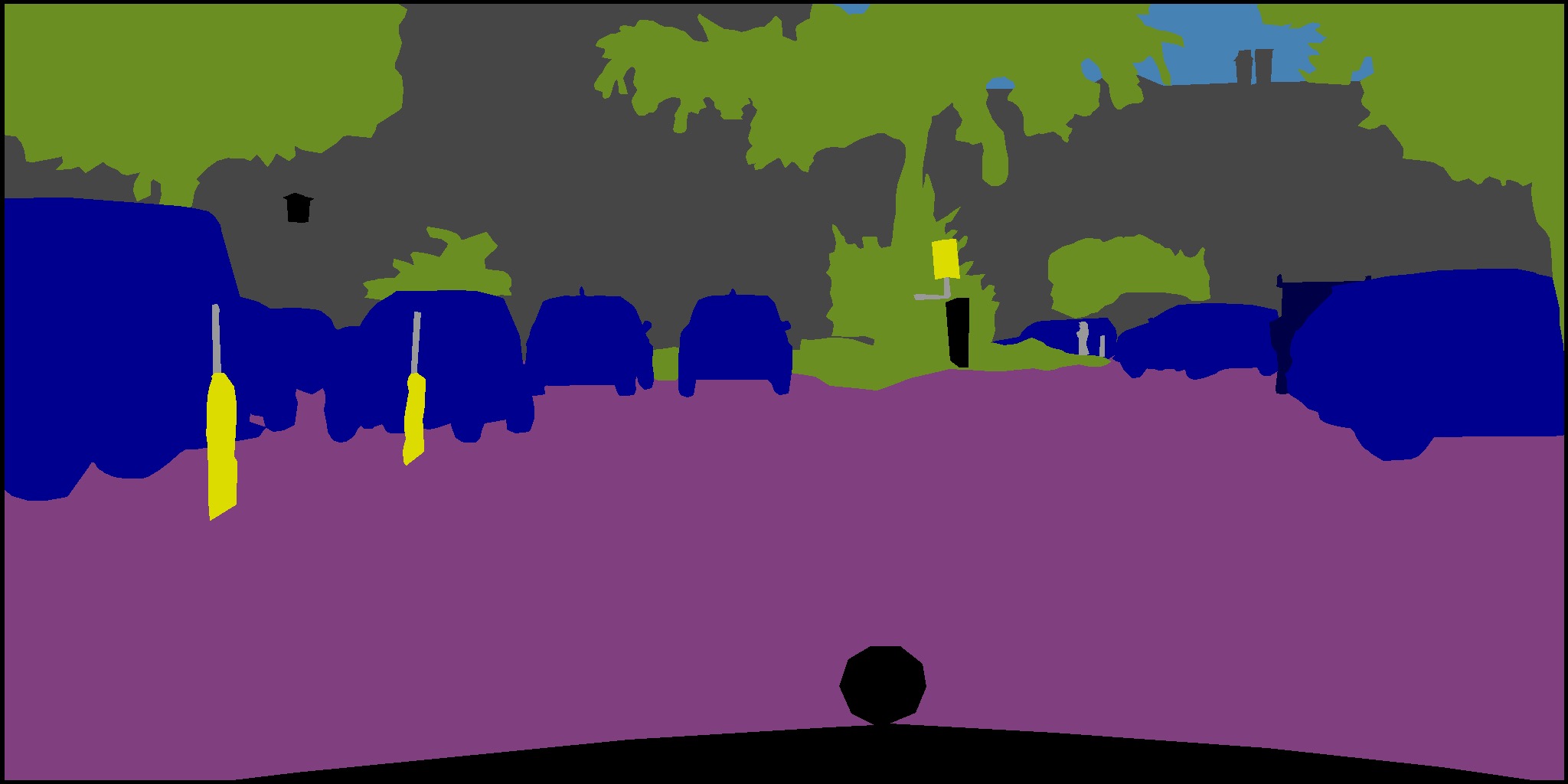}
	\end{minipage}}
	\caption{Qualitative results for domain adaptation segmentation. The samples are randomly selected from the validation subsets of Cityscapes.}
	\label{fig:seg}
\end{figure}

\begin{figure*}[h]
	\centering
	\subfigure[Source domain (\emph{GTA}: video game images).]{
		\begin{minipage}[b]{0.48\linewidth}
			\includegraphics[width=0.32\linewidth]{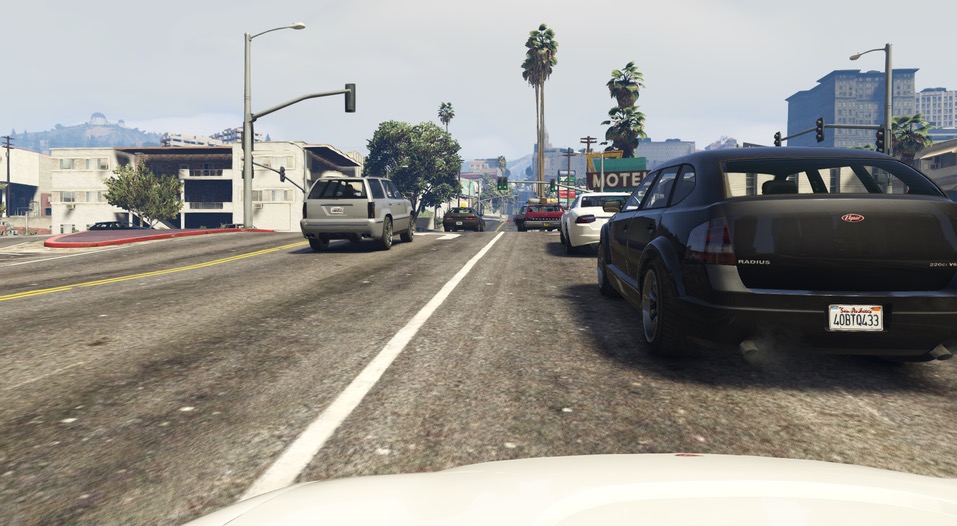}
			\includegraphics[width=0.32\linewidth]{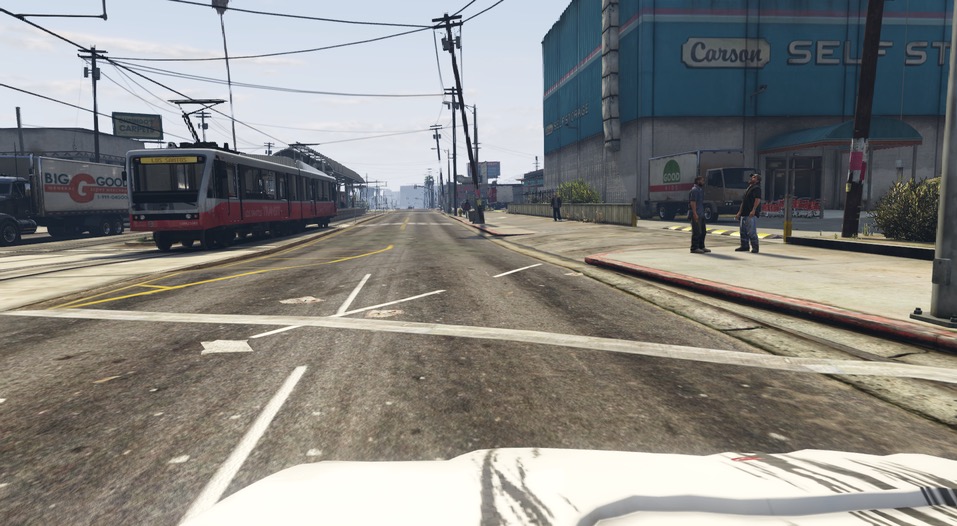}
			\includegraphics[width=0.32\linewidth]{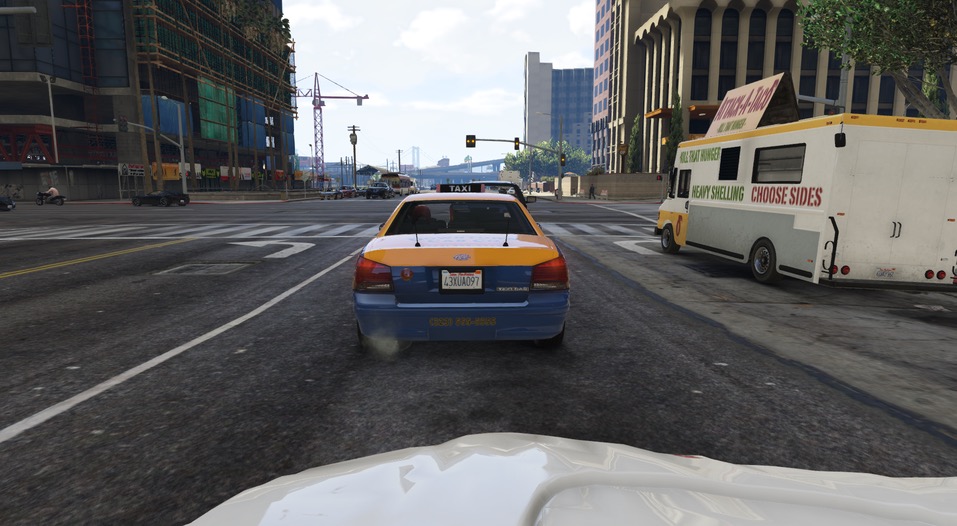}
	\end{minipage}}
	\subfigure[Target domain (Cityscapes: real-world images).]{
		\begin{minipage}[b]{0.48\linewidth}
			\includegraphics[width=0.32\linewidth]{./fig/sup/o_001.jpg}
			\includegraphics[width=0.32\linewidth]{./fig/sup/o_002.jpg}
			\includegraphics[width=0.32\linewidth]{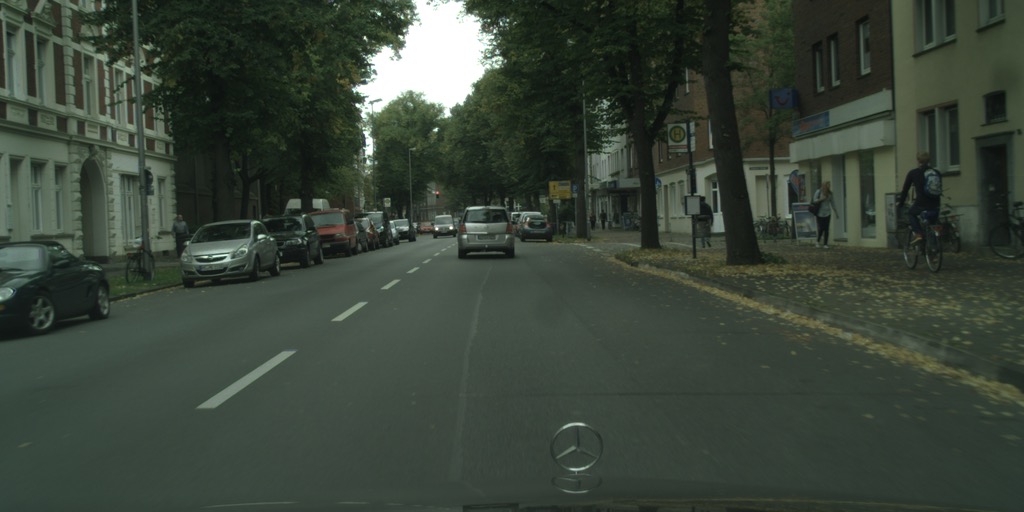}
	\end{minipage}}
	
	\caption{Dataset samples for segmentation experiments (video games $\rightarrow$ street views). }
	\label{fig:gta}
\end{figure*}

\section{Related Work}
Recent achievements on domain-invariant learning generally follow two directions.
The first direction is learning a single network,which is encouraged to produce domain-invariant features by minimizing additional loss functions in the network training \cite{ganin2016domain,long2015learning,long2016deep,long2016unsupervised,tzeng2014deep}. 
The Maximum Mean Discrepancy (MMD) \cite{gretton2007kernel}, and MK-MMD \cite{gretton2012optimal} in \cite{long2015learning}, are adopted as the discrepancy metric among domains.
Beyond the first order statistic, second-order statistics are utilized in \cite{hoffman2014lsda}.
Besides the hand-crafted distribution distance metrics, \cite{ganin2016domain,tzeng2015simultaneous,long2018conditional} resort to adversarial training and achieve superior performances.
Various distribution alignment methods, e.g., \cite{wu2019domain,kumar2018co}, are proposed to improve the invariant feature learning.
While effective in certain scenarios, the performance of learning invariant features using a shared network is largely constrained by the degree of domain shift as discussed in \cite{rozantsev2018residual}.
Meanwhile, some recent works like \cite{wu2019domain,zhao2019learning} suggest important insights on whether it is sufficient to do domain adaptation by invariant representation and small empirical source risk, which shed light on exploring more effective alignment methods that are robust to common issues like different marginal label distributions. 
Another popular direction is modeling each domain explicitly using auxiliary network structures.
\cite{bousmalis2016domain} proposes feature representation by two components where domain similarities and shifts are modeled by a private component and a shared component separately.
A completely two-stream network structure is proposed in \cite{rozantsev2018residual}, where auxiliary residual networks are trained to adapt the layer parameters of the source to the target domain.
\cite{chang2019domain} proposes attacking domain shifts by domain-specific batch normalization, which we believe is compatible with the proposed DAFD for better performance. 
Another popular direction for domain adaptation is to remap the input data between the source and the target domain for domain adaptation \cite{murez2017image,hu2018duplex,hoffman2018cycada}, which is not included in the discussion since we are focusing on learning invariant feature space.
Also, as discussed in \cite{rozantsev2018residual}, while remarkable performances are witnessed by adopting pseudo-labels \cite{long2015learning,saito2017asymmetric,zhang2018collaborative}, we consider adopting pseudo-labels as a plug-and-play improvement that can be equipped to our method, but does not align with the main focus of our research.
Finally, learning invariance is of relevance beyond domain adaptation, e.g., in the field of causal inference \cite{buhlmann2018invariance}.

\section{Conclusion}
We proposed to perform domain-invariant learning through domain-adaptive filter decomposition.
To model domain shifts, convolutional filters in a deep convolutional network are decomposed over domain-adaptive dictionary atoms to counter domain shifts, and cross-domain decomposition coefficients are constrained to unify common semantics.
We present the intuitions of countering domain shifts by adapting atoms through toy examples, and further provide theoretical analysis.
Extensive experiments on multiple tasks and network architectures with significant improvements validate that, by stacking domain-adaptive branched layers with filter decomposition, complex domain shifts in real-world scenarios can be bridged to produce domain-invariant representation, which are reflected by both experimental results and feature space visualizations, all this at virtual no additional memory or computational cost when adding domains.

\section{Broader Impact}
In this paper, we introduced a plug-in framework to explicitly model domain shifts in CNNs. With the proposed architecture, we need only a small set of dictionary atoms to model each additional domain, which brings a negligible amount of additional parameters, typically a few hundred. We consider our plug-and-play method a general contribution to deep learning, assuming no particular application.

\section{Acknowledgements}
Work partially supported by NSF, NGA, ARO, ONR, and gifts from Cisco, Google, Amazon, and Microsoft.

\bibliographystyle{plain}
\bibliography{egbib}

\begin{thebibliography}{10}

\bibitem{bousmalis2016domain}
Konstantinos Bousmalis, George Trigeorgis, Nathan Silberman, Dilip Krishnan,
  and Dumitru Erhan.
\newblock Domain separation networks.
\newblock In {\em Advances in Neural Information Processing Systems}, 2016.

\bibitem{buhlmann2018invariance}
Peter B{\"u}hlmann.
\newblock Invariance, causality and robustness.
\newblock {\em arXiv preprint arXiv:1812.08233}, 2018.

\bibitem{chang2019domain}
Woong-Gi Chang, Tackgeun You, Seonguk Seo, Suha Kwak, and Bohyung Han.
\newblock Domain-specific batch normalization for unsupervised domain
  adaptation.
\newblock In {\em IEEE Conference on Computer Vision and Pattern Recognition},
  2019.

\bibitem{cordts2016cityscapes}
Marius Cordts, Mohamed Omran, Sebastian Ramos, Timo Rehfeld, Markus Enzweiler,
  Rodrigo Benenson, Uwe Franke, Stefan Roth, and Bernt Schiele.
\newblock The cityscapes dataset for semantic urban scene understanding.
\newblock In {\em IEEE Conference on Computer Vision and Pattern Recognition},
  2016.

\bibitem{ganin2016domain}
Yaroslav Ganin, Evgeniya Ustinova, Hana Ajakan, Pascal Germain, Hugo
  Larochelle, Fran{\c{c}}ois Laviolette, Mario Marchand, and Victor Lempitsky.
\newblock Domain-adversarial training of neural networks.
\newblock {\em The Journal of Machine Learning Research}, 17(1):2096--2030,
  2016.

\bibitem{gretton2007kernel}
Arthur Gretton, Karsten~M Borgwardt, Malte Rasch, Bernhard Sch{\"o}lkopf, and
  Alex~J Smola.
\newblock A kernel method for the two-sample-problem.
\newblock In {\em Advances in Neural Information Processing Systems}, 2007.

\bibitem{gretton2012optimal}
Arthur Gretton, Dino Sejdinovic, Heiko Strathmann, Sivaraman Balakrishnan,
  Massimiliano Pontil, Kenji Fukumizu, and Bharath~K Sriperumbudur.
\newblock Optimal kernel choice for large-scale two-sample tests.
\newblock In {\em Advances in Neural Information Processing Systems}, 2012.

\bibitem{he2016deep}
Kaiming He, Xiangyu Zhang, Shaoqing Ren, and Jian Sun.
\newblock Deep residual learning for image recognition.
\newblock In {\em IEEE Conference on Computer Vision and Pattern Recognition},
  2016.

\bibitem{hoffman2014lsda}
Judy Hoffman, Sergio Guadarrama, Eric~S Tzeng, Ronghang Hu, Jeff Donahue, Ross
  Girshick, Trevor Darrell, and Kate Saenko.
\newblock {LSDA}: Large scale detection through adaptation.
\newblock In {\em Advances in Neural Information Processing Systems}, 2014.

\bibitem{hoffman2018cycada}
Judy Hoffman, Eric Tzeng, Taesung Park, Jun-Yan Zhu, Phillip Isola, Kate
  Saenko, Alexei Efros, and Trevor Darrell.
\newblock {CyCADA}: Cycle-consistent adversarial domain adaptation.
\newblock In {\em International Conference on Machine Learning}, 2018.

\bibitem{hoffman2016fcns}
Judy Hoffman, Dequan Wang, Fisher Yu, and Trevor Darrell.
\newblock {FCN}s in the wild: Pixel-level adversarial and constraint-based
  adaptation.
\newblock {\em arXiv preprint arXiv:1612.02649}, 2016.

\bibitem{hu2018duplex}
Lanqing Hu, Meina Kan, Shiguang Shan, and Xilin Chen.
\newblock Duplex generative aaversarial network for unsupervised domain
  qdaptation.
\newblock In {\em IEEE Conference on Computer Vision and Pattern Recognition},
  2018.

\bibitem{krizhevsky2012imagenet}
Alex Krizhevsky, Ilya Sutskever, and Geoffrey~E Hinton.
\newblock Imagenet classification with deep convolutional neural networks.
\newblock In {\em Advances in Neural Information Processing Systems}, 2012.

\bibitem{kumar2018co}
Abhishek Kumar, Prasanna Sattigeri, Kahini Wadhawan, Leonid Karlinsky, Rogerio
  Feris, Bill Freeman, and Gregory Wornell.
\newblock Co-regularized alignment for unsupervised domain adaptation.
\newblock In {\em Advances in Neural Information Processing Systems}, pages
  9345--9356, 2018.

\bibitem{li2013casia}
Stan Li, Dong Yi, Zhen Lei, and Shengcai Liao.
\newblock The casia nir-vis 2.0 face database.
\newblock In {\em IEEE Conference on Computer Vision and Pattern Recognition
  Workshops}, 2013.

\bibitem{long2015learning}
Mingsheng Long, Yue Cao, Jianmin Wang, and Michael~I Jordan.
\newblock Transferable representation learning with deep adaptation networks.
\newblock {\em IEEE Transactions on Pattern Analysis and Machine Intelligence},
  41(12):3071--3085, 2019.

\bibitem{long2018conditional}
Mingsheng Long, Zhangjie Cao, Jianmin Wang, and Michael~I Jordan.
\newblock Conditional adversarial domain adaptation.
\newblock In {\em Advances in Neural Information Processing Systems}, 2018.

\bibitem{long2016unsupervised}
Mingsheng Long, Han Zhu, Jianmin Wang, and Michael~I Jordan.
\newblock Unsupervised domain adaptation with residual transfer networks.
\newblock In {\em Advances in Neural Information Processing Systems}, 2016.

\bibitem{long2016deep}
Mingsheng Long, Han Zhu, Jianmin Wang, and Michael~I Jordan.
\newblock Deep transfer learning with joint adaptation networks.
\newblock {\em International Conference on Machine Learning}, 2017.

\bibitem{maaten2008visualizing}
Laurens van~der Maaten and Geoffrey Hinton.
\newblock Visualizing data using t-sne.
\newblock {\em Journal of Machine Learning Research}, 9(Nov):2579--2605, 2008.

\bibitem{murez2017image}
Zak Murez, Soheil Kolouri, David Kriegman, Ravi Ramamoorthi, and Kyungnam Kim.
\newblock Image to image translation for domain adaptation.
\newblock {\em IEEE Conference on Computer Vision and Pattern Recognitio},
  2018.

\bibitem{qiu2018dcfnet}
Qiang Qiu, Xiuyuan Cheng, Robert Calderbank, and Guillermo Sapiro.
\newblock {DCFNet}: Deep neural network with decomposed convolutional filters.
\newblock {\em International Conference on Machine Learning}, 2018.

\bibitem{richter2016playing}
Stephan~R Richter, Vibhav Vineet, Stefan Roth, and Vladlen Koltun.
\newblock Playing for data: Ground truth from computer games.
\newblock In {\em European Conference on Computer Vision}, 2016.

\bibitem{rozantsev2018beyond}
Artem Rozantsev, Mathieu Salzmann, and Pascal Fua.
\newblock Beyond sharing weights for deep domain adaptation.
\newblock {\em IEEE Transactions on Pattern Analysis and Machine Intelligence},
  2018.

\bibitem{rozantsev2018residual}
Artem Rozantsev, Mathieu Salzmann, and Pascal Fua.
\newblock Residual parameter transfer for deep domain adaptation.
\newblock In {\em IEEE Conference on Computer Vision and Pattern Recognition},
  2018.

\bibitem{saenko2010adapting}
Kate Saenko, Brian Kulis, Mario Fritz, and Trevor Darrell.
\newblock Adapting visual category models to new domains.
\newblock In {\em European Conference on Computer Vision}, 2010.

\bibitem{saito2017asymmetric}
Kuniaki Saito, Yoshitaka Ushiku, and Tatsuya Harada.
\newblock Asymmetric tri-training for unsupervised domain adaptation.
\newblock {\em International Conference on Machine Learning}, 2017.

\bibitem{simonyan2014very}
Karen Simonyan and Andrew Zisserman.
\newblock Very deep convolutional networks for large-scale image recognition.
\newblock {\em International Conference on Learning Representations}, 2014.

\bibitem{sulam2018multilayer}
Jeremias Sulam, Vardan Papyan, Yaniv Romano, and Michael Elad.
\newblock Multilayer convolutional sparse modeling: Pursuit and dictionary
  learning.
\newblock {\em IEEE Transactions on Signal Processing}, 66(15):4090--4104,
  2018.

\bibitem{tsai2018learning}
Yi-Hsuan Tsai, Wei-Chih Hung, Samuel Schulter, Kihyuk Sohn, Ming-Hsuan Yang,
  and Manmohan Chandraker.
\newblock Learning to adapt structured output space for semantic segmentation.
\newblock In {\em IEEE Conference on Computer Vision and Pattern Recognition},
  2018.

\bibitem{tzeng2015simultaneous}
Eric Tzeng, Judy Hoffman, Trevor Darrell, and Kate Saenko.
\newblock Simultaneous deep transfer across domains and tasks.
\newblock In {\em IEEE International Conference on Computer Vision}, 2015.

\bibitem{tzeng2017adversarial}
Eric Tzeng, Judy Hoffman, Kate Saenko, and Trevor Darrell.
\newblock Adversarial discriminative domain adaptation.
\newblock In {\em IEEE Conference on Computer Vision and Pattern Recognition},
  2017.

\bibitem{tzeng2014deep}
Eric Tzeng, Judy Hoffman, Ning Zhang, Kate Saenko, and Trevor Darrell.
\newblock Deep domain confusion: Maximizing for domain invariance.
\newblock {\em arXiv preprint arXiv:1412.3474}, 2014.

\bibitem{wu2019domain}
Yifan Wu, Ezra Winston, Divyansh Kaushik, and Zachary Lipton.
\newblock Domain adaptation with asymmetrically-relaxed distribution alignment.
\newblock In {\em International Conference on Machine Learning}, 2019.

\bibitem{zhang2018collaborative}
Weichen Zhang, Wanli Ouyang, Wen Li, and Dong Xu.
\newblock Collaborative and adversarial network for unsupervised domain
  adaptation.
\newblock In {\em IEEE Conference on Computer Vision and Pattern Recognition},
  2018.

\bibitem{zhao2019learning}
Han Zhao, Remi~Tachet Des~Combes, Kun Zhang, and Geoffrey Gordon.
\newblock On learning invariant representations for domain adaptation.
\newblock In {\em International Conference on Machine Learning}, 2019.

\end{thebibliography}

\clearpage
\appendix

\setcounter{table}{0}
\renewcommand{\thetable}{A.\arabic{table}}

\setcounter{table}{0}
\renewcommand{\thetable}{A.\arabic{table}}

\setcounter{equation}{0}
\renewcommand{\theequation}{A.\arabic{equation}}

\setcounter{figure}{0}
\renewcommand{\thefigure}{A.\arabic{figure}}

\section{Toy Experiment}

\begin{figure}[h]
	\centering
	\includegraphics[width=0.8\linewidth]{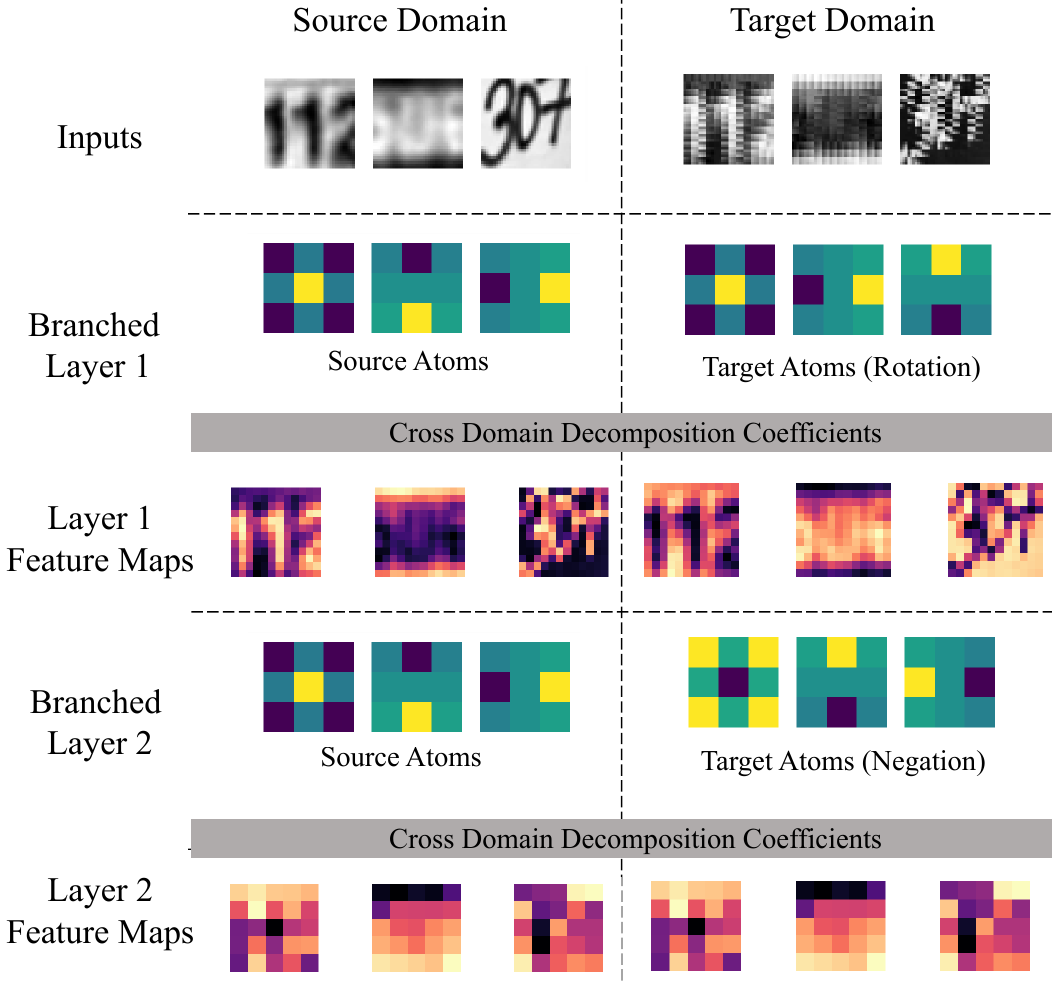}\\
	\caption{Visualization of the toy example. The two columns visualize the inputs, features, and atoms of the source domain and the target domain, respectively. Only the output feature in the first channel of each convolutional layer is visualized for comparison. Domain invariant features, the last row, are obtained by manually adapting source domain atoms to generated target domain atoms.}	
	\label{fig:rn} 
\end{figure}

\section{Dataset Samples and Qualitative Results}
\subsection{Unsupervised DA for Image Segmentation}
\label{imseg}
For the image segmentation experiments in Section {\color{red}{5.2}}, we provide more qualitative results in Figure~\ref{fig:segA}.

\begin{figure*}[h]
	\centering
	\subfigure[Target domain image.]{
		\begin{minipage}[b]{0.23\linewidth}
			\includegraphics[width=1\linewidth]{./fig/sup/o_003.jpg}\vspace{3pt}
			\includegraphics[width=1\linewidth]{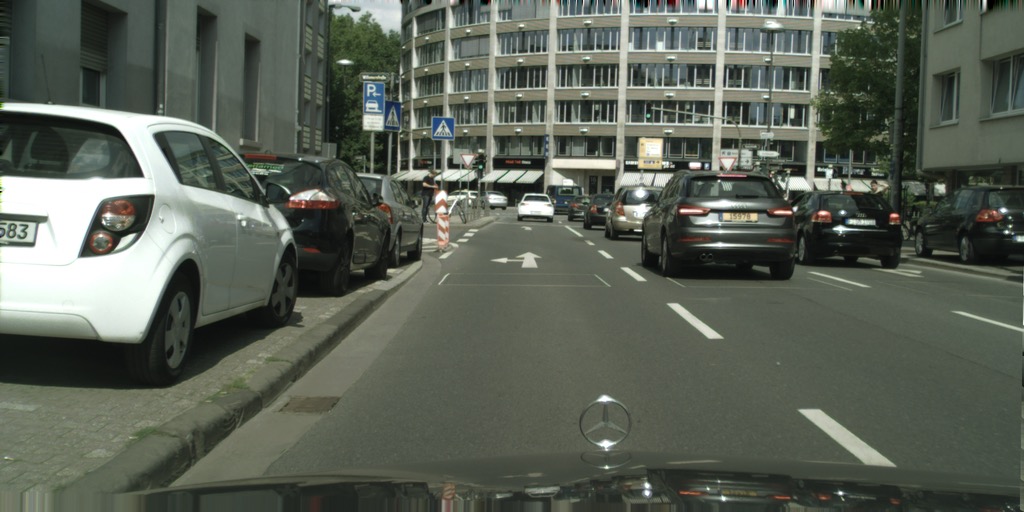}\vspace{3pt}
			\includegraphics[width=1\linewidth]{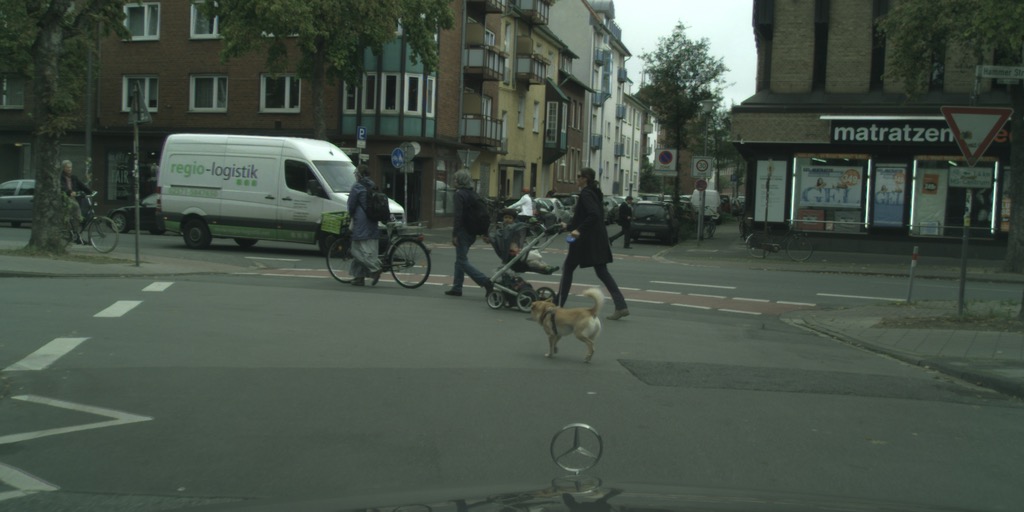}\vspace{3pt}
			\includegraphics[width=1\linewidth]{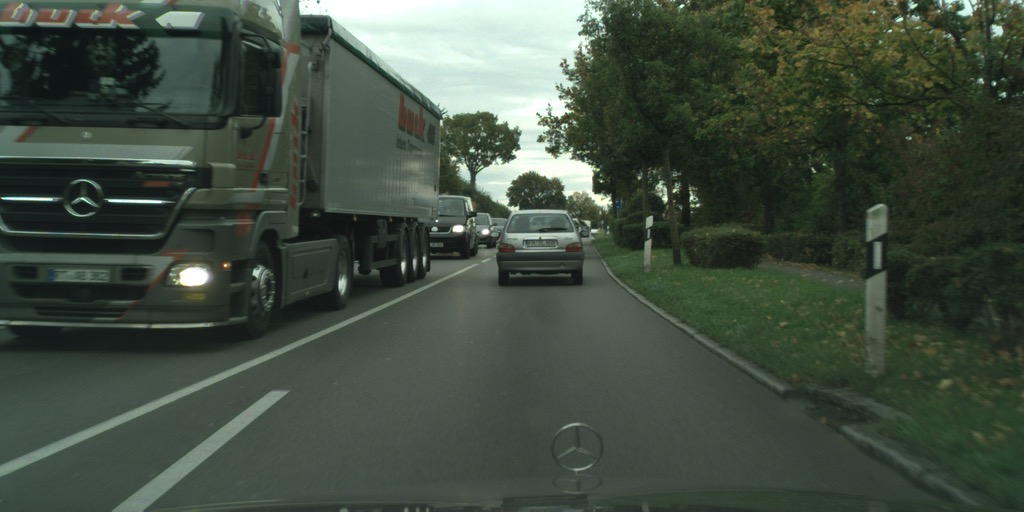}
	\end{minipage}}
	\subfigure[Before adaptation.]{
		\begin{minipage}[b]{0.23\linewidth}
			\includegraphics[width=1\linewidth]{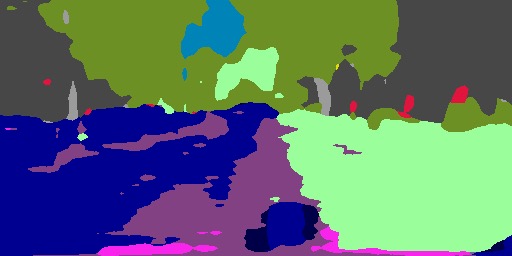}\vspace{3pt}
			\includegraphics[width=1\linewidth]{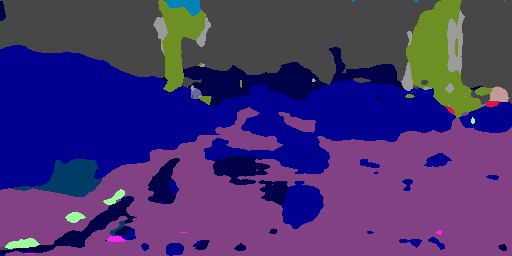}\vspace{3pt}
			\includegraphics[width=1\linewidth]{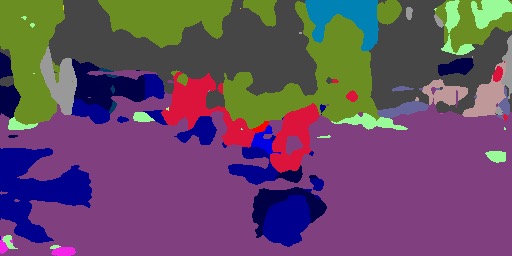}\vspace{3pt}
			\includegraphics[width=1\linewidth]{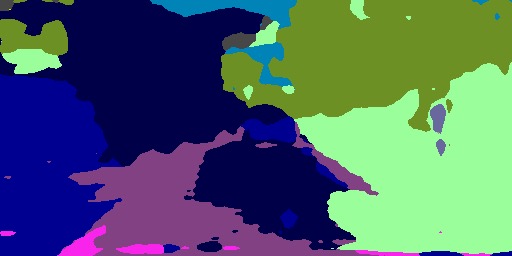}
	\end{minipage}}
	\subfigure[After adaptation.]{
		\begin{minipage}[b]{0.23\linewidth}
			\includegraphics[width=1\linewidth]{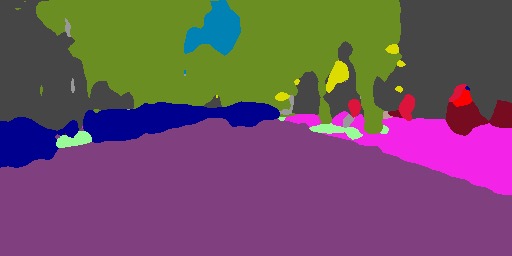}\vspace{3pt}
			\includegraphics[width=1\linewidth]{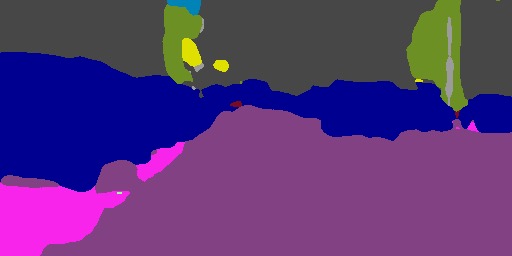}\vspace{3pt}
			\includegraphics[width=1\linewidth]{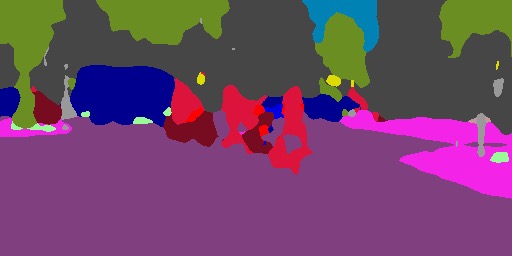}\vspace{3pt}
			\includegraphics[width=1\linewidth]{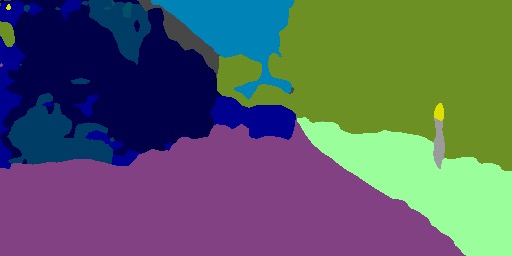}
			
	\end{minipage}}
	\subfigure[Ground truth.]{
		\begin{minipage}[b]{0.23\linewidth}
			\includegraphics[width=1\linewidth]{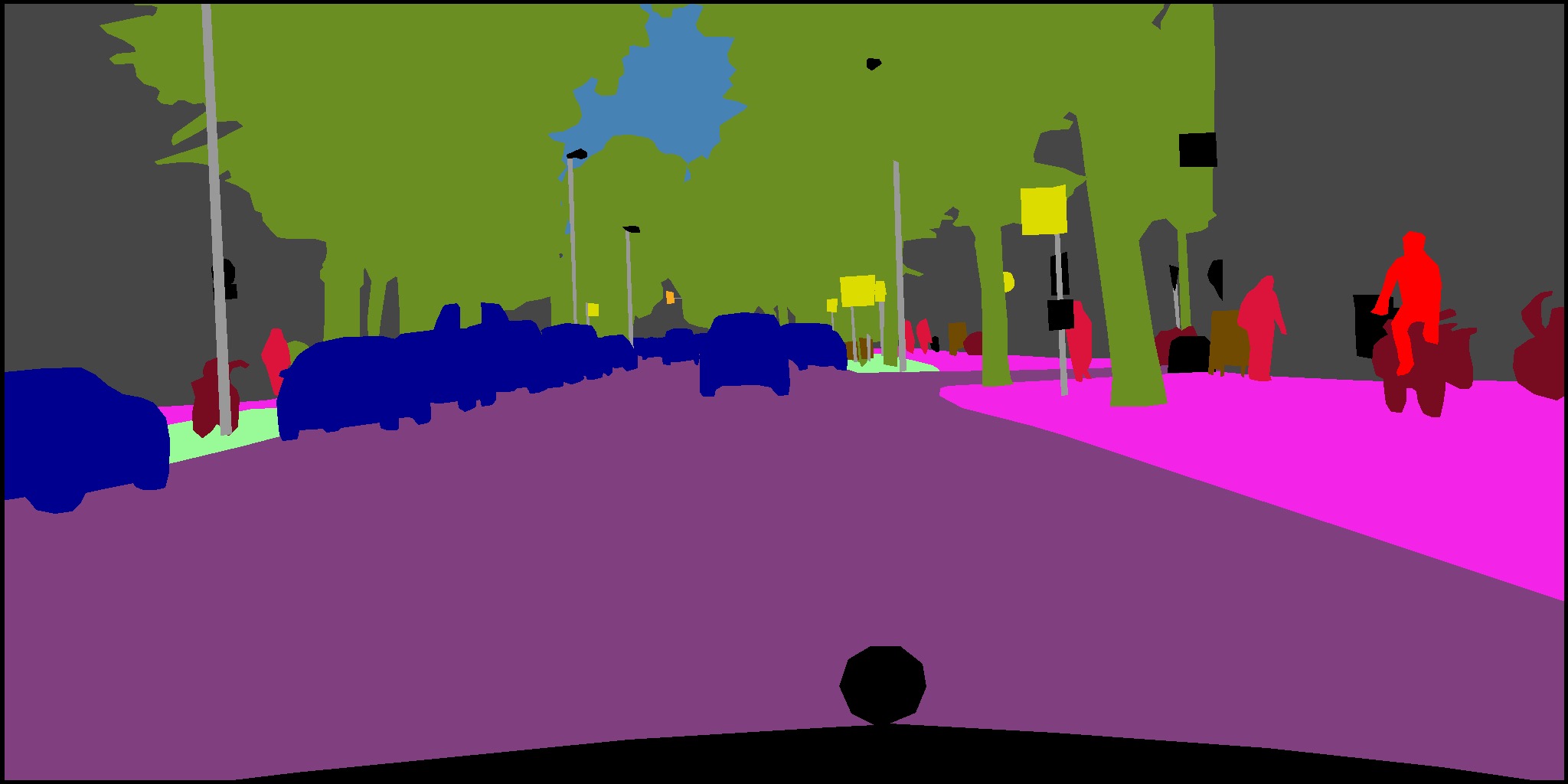}\vspace{3pt}
			\includegraphics[width=1\linewidth]{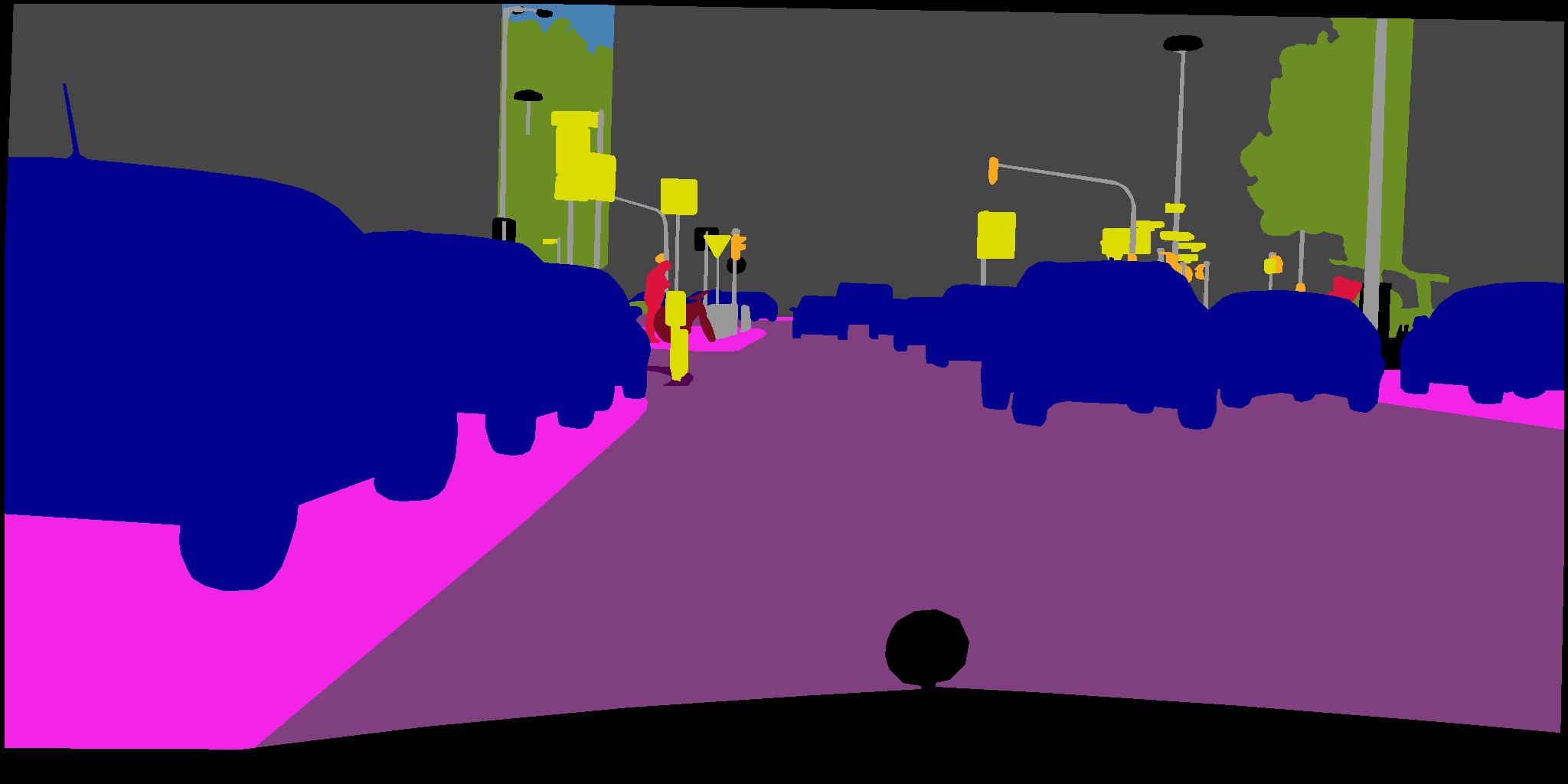}\vspace{3pt}
			\includegraphics[width=1\linewidth]{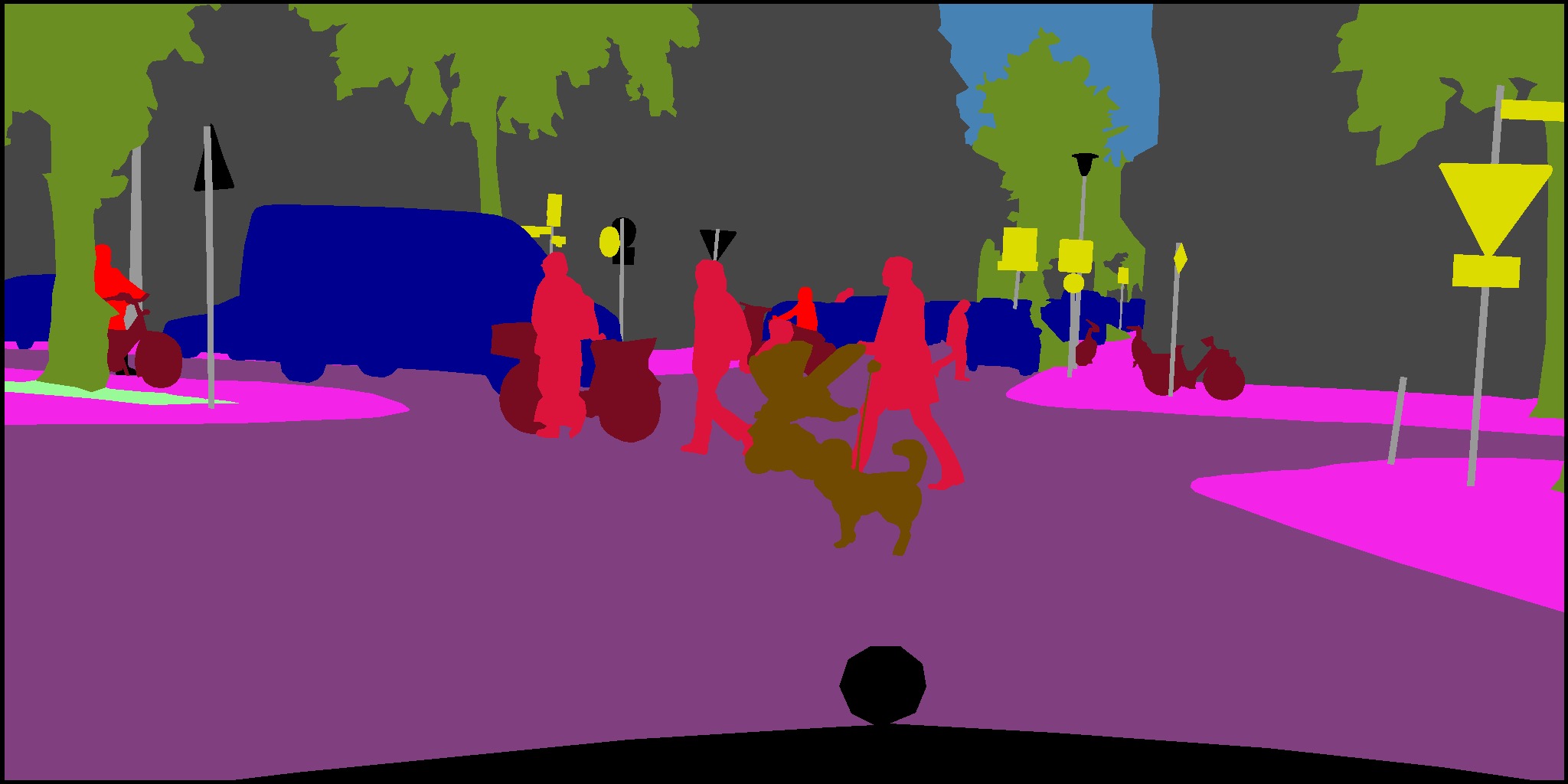}\vspace{3pt}
			\includegraphics[width=1\linewidth]{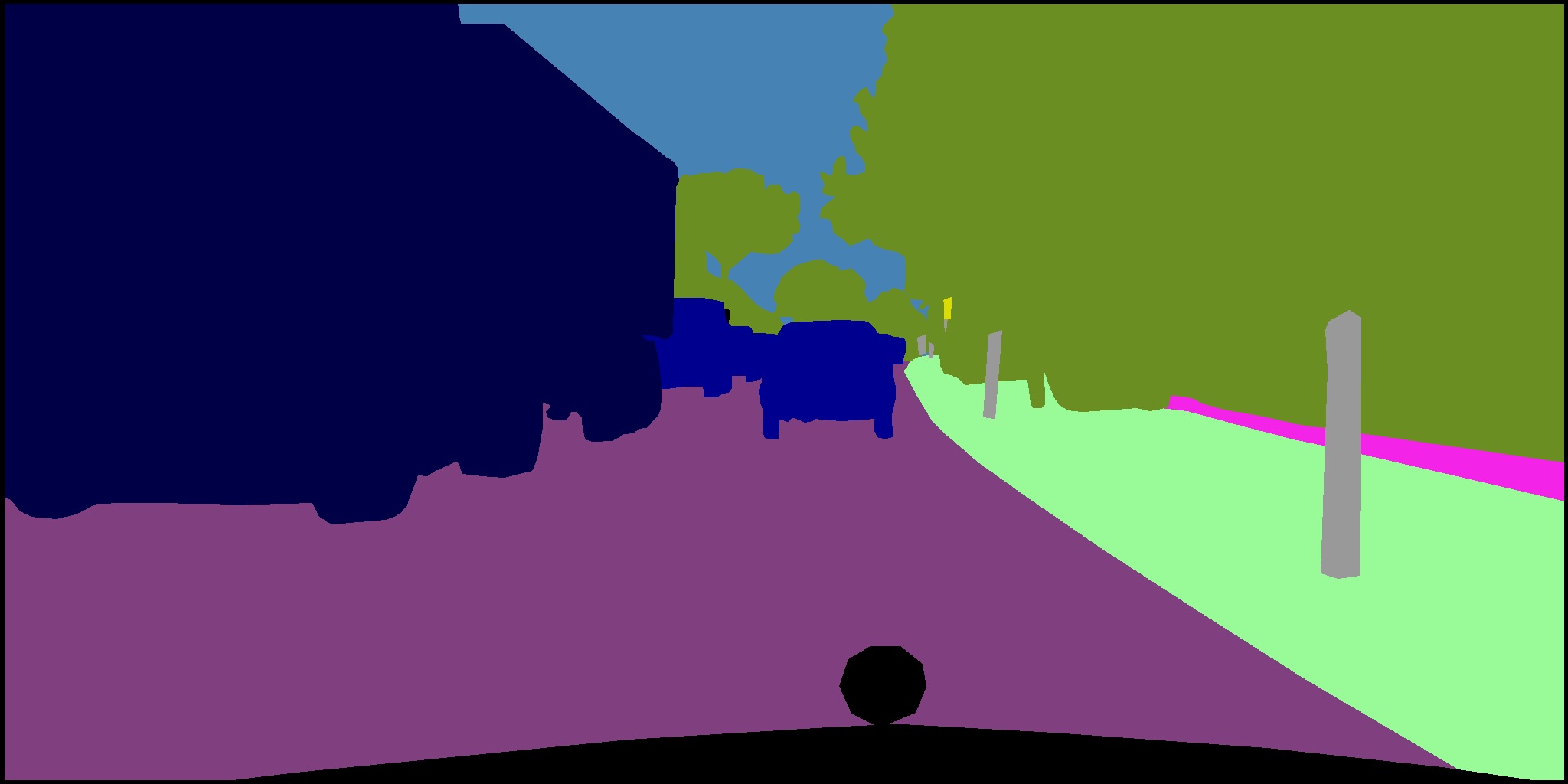}
	\end{minipage}}
	\caption{Qualitative results for domain adaptation segmentation. The samples are randomly selected from the validation subsets of Cityscapes. }
	\label{fig:segA}
\end{figure*}

\section{Computation and Parameters}
In Table~\ref{tab:cost}, we provide comparisons on additional parameters and computation introduced by one extra domain with and without the proposed domain-adaptive filter decomposition. The comparison reveals that domain-adaptive filter decomposition not only delivers superior performances but also saves both parameters and computation significantly.

\begin{table}[h]
	\begin{center} 
		\caption{Comparisons on additional parameters and computation introduced by one extra domain. Comparisons are performed on VGG-16, with 6 dictionary atoms and the input size of 224 $\times$ 224.}
		\vspace{2mm}
		\label{tab:cost}
		\begin{tabular}{c|c c }
			\toprule
			
			Model & Regular VGG & VGG with DAFD  \\
			\midrule
			Parameters & 14.71M & 0.0007M  \\
			Flops & 15.38G & 10.75G  \\
			
			\bottomrule
		\end{tabular}
	\end{center}
\end{table}

\section{Correction of a Single Filter Transform}
We first analyze the ``symmetric'' correction of one filter spatial transform $D_\tau$ in one layer.
The inclusion of linear correspondence transform is more direct. 
For technical reasons, 
we assume that the displacement field  $\tau$ is a small distortion,
namely $\| \nabla \tau \|_\infty \ll 1$,
and then $D_\tau$ is invertible.
Example includes rotation by a small angle and a small factor rescaling (dilation). 

For simplicity we only consider one input and output channel in each of the multiple convolutional layers.
The argument extends to multiple channels by modifying the boundedness condition of the filters. 
Then the forward mapping in one convolutional layer can be written as 
$
y = \sigma( x \ast w + b),
$
where $x$ is the input activation, $y$ is the output, $w$ is the filter, $b$ is the constant bias,
and $\sigma$ is the nonlinear activation function, e.g., ReLU. 
As we take a continuous formulation in the analysis,
the activations $x$ and $y$ are assumed to be smooth functions supported on domain $\Omega \subset \R^2$,
typically $\Omega = [-1,1]^2$. 
The filter $w$ is a function supported on $2^{j} B$,
$B$ being the unit disk, and $2^{j}$ is layer scale (diameter of filter patches) .
The 1-norm of a function is defined to be 
$\| x \|_1  = \int_{\R^2} |x(u)| du$.

\begin{lemma}
	\label{lemma:commute}
	Suppose that the two filters $w$,
	$f$ are supported on $2^{j_w}B$ and $2^{j_f}B$ respectively.
	$\sigma: \R \to \R$ is non-expansive,
	$D_\tau$ is a spatial transform where $\tau$ is odd, i.e., $\tau(-u) = -\tau(u)$,
	and $|\nabla \tau|_\infty < \frac{1}{5}$.
	Then\begin{equation*} 
		\begin{split}
			&	
			\| \sigma_b ( x \ast D_\tau w) \ast f 
			-
			\sigma_b ( x \ast w) \ast D_\tau^{-1} f \|_1 \\
			&
			~~~ \le 2 |\nabla \tau|_\infty \|w\|_1 \|f \|_1 \left\{ (2^{j_w} + 2^{j_f}) \| \nabla x\|_1 + 4 \|x\|_1 \right\},
		\end{split}
	\end{equation*}
	where
	$\sigma_b$ denotes the nonlinear function with the bias.
	The second term vanishes if $(I_d - \tau)$ is a rigid motion, e.g., rotation.
\end{lemma}


\begin{proof}[Proof of Lemma \ref{lemma:commute}]
	We establish a few facts:
	
	Fact 1. $|\nabla \tau|_\infty < \frac{1}{5}$ guarantees that, $\rho := I_d - \tau$,
	\begin{equation}\label{eq:Jrho-bound}
		| |J\rho | -1|,  | |J\rho^{-1} | -1| \le 4  |\nabla \tau|_\infty,
	\end{equation}
	where $Jf = \det(\nabla f)$ denotes the determinate of the Jacobian matrix of  the mapping $f:\R^2 \to \R^2$.
	The inequality can be verified by elementary calculation. 
	When $\rho$ is a rigid motion then the r.h.s of \eqref{eq:Jrho-bound} is zero. 
	
	Fact 2. $\rho$ is invertible, and odd symmetry of $\tau$ implies that $\rho$ and thus $\rho^{-1}$ are odd,
	namely $-\rho^{-1}(-u) = \rho^{-1}(u)$.
	
	Define
	\begin{align*}
		y_1(u) 
		& := \sigma_b ( x \ast D_\tau w) \ast f(u) \\
		& =  \int_{\R^2} \sigma_b \left(  \int_{\R^2} x(u+v - z ) w(\rho(z)) dz \right) f(-v)dv \\
		& =  \int_{\R^2} \sigma_b \left(  \int_{\R^2} x(u+v - \rho^{-1}( \tilde{z}) ) w(\tilde{z}) |J \rho^{-1}(\tilde{z})| d\tilde{z} \right) f(-v)dv
	\end{align*}
	and
	\[
	\hat{y}_1(u) 
	:=
	\int_{\R^2} \sigma_b \left(  \int_{\R^2} x(u+v - \rho^{-1}( \tilde{z}) ) w(\tilde{z}) d\tilde{z} \right) f(-v)dv.
	\]
	We have that 
	\begin{align*}
		|y_1(u) - \hat{y}_1(u) |
		& \le 
		\int_{\R^2} \int_{\R^2} |x(u+v - \rho^{-1}( \tilde{z}) )| |w(\tilde{z})| 
		\left|  |J \rho^{-1}|-1 \right|  |f(-v)| d\tilde{z} dv 
		~ \text{ 
			(by $\sigma_b$ non-expansive)}  \\
		&  \le
		4  |\nabla \tau|_\infty
		\int_{\R^2} \int_{\R^2} |x(u+v - \rho^{-1}( \tilde{z}) )| |w(\tilde{z})| 
		|f(-v)| d\tilde{z} dv 
		~ \text{(by Fact 1)}  
	\end{align*}
	and thus
	\begin{equation}\label{eq:y1-y1h}
		\| y_1 - \hat{y}_1 \|_1
		\le 
		4  |\nabla \tau|_\infty 
		\|x\|_1
		\|w\|_1 \|f\|_1.
	\end{equation}
	When $\rho$ is a rigid motion, $y_1 = \hat{y}_1$.

	Also, let
	\begin{align*}
		y_2(u) 
		& := \sigma_b ( x \ast w) \ast D_\tau^{-1} f (u)\\
		& =  \int_{\R^2}  \sigma_b \left(  \int_{\R^2} x (u+v-z) w(z)dz \right) f ( - \rho^{-1}(v)) dv ~~~\text{ (by Fact 2)} \\
		& = \int_{\R^2}  \sigma_b \left(  \int_{\R^2} x (u+ \rho(\tilde{v}) -z) w(z)dz \right) f ( - \tilde{v}) |J\rho(\tilde{v})| d\tilde{v}
	\end{align*}
	and
	\[
	\hat{y}_2(u) 
	:= \int_{\R^2}  \sigma_b \left(  \int_{\R^2} x (u+ \rho(\tilde{v}) -z) w(z)dz \right) f ( - \tilde{v}) d\tilde{v}.
	\]
	Similar to the proof of \eqref{eq:y1-y1h}, 
	one can verify that 
	\begin{equation}\label{eq:y2-y2h}
		\| y_2 - \hat{y}_2 \|_1
		\le 
		4  |\nabla \tau|_\infty 
		\|x\|_1
		\|w\|_1 \|f\|_1,
	\end{equation}
	and the bound is zero when $\rho$ is a rigid motion. 
	
	It remains to bound
	$\|  \hat{y}_1 - \hat{y}_2 \|_1$.
	Note that by $\sigma_b$ being non-expansive again
	\begin{equation}\label{eq:bound1}
		| \hat{y}_1(u) - \hat{y}_2(u) |
		\le
		\int_{\R^2}   \int_{\R^2} | x(u+v - \rho^{-1}(z) )  -  x (u+ \rho(v) -z) | |w(z)| dz |f(-v)|dv.
	\end{equation}
	We claim that 
	\begin{equation}\label{eq:uniform-bound-x-diff}
		\int_{\R^2} | x(u+v - \rho^{-1}(z) )  -  x (u+ \rho(v) -z) | du
		\le |\nabla \tau|_\infty 2 (2^{j_w} + 2^{j_f})  \| \nabla x \|_1
	\end{equation}
	uniformly for $v$ and $z$. 
	If true,
	with \eqref{eq:bound1} it gives that
	\[
	\int_{\R^2} | \hat{y}_1(u) - \hat{y}_2(u) |du
	\le |\nabla \tau|_\infty    2 (2^{j_w} + 2^{j_f}) \| \nabla x\|_1 \|w\|_1 \|f\|_1
	\]
	which proves the lemma together with \eqref{eq:y1-y1h} and \eqref{eq:y2-y2h}.
	
	Proof of \eqref{eq:uniform-bound-x-diff}:
	We verify that for any fixed $v$, $z$, 
	\begin{equation}\label{eq:xu0-xu1-bound}
		\int_{\R^2} | x(u+v - \rho^{-1}(z) )  -  x (u+ \rho(v) -z) | du 
		\le \|\nabla x\|_1 |\nabla \tau|_\infty |v- \rho^{-1}(z) |,
	\end{equation}
	by a direct calculation:
	\begin{align*}
		\text{(l.h.s)} 
		& \le \|\nabla x\|_1 | (v - \rho^{-1}(z) )-  (\rho(v) -z) | \\
		& =  \|\nabla x\|_1 | \tau(v) - \tau( \rho^{-1}(z) )  | \\
		& \le  \|\nabla x\|_1 |\nabla \tau|_\infty |v- \rho^{-1}(z) |. 
	\end{align*}
	Then, combined with that $v \in 2^{j_f} B$ thus $|v| \le 2^{j_f}$, 
	and $z \in 2^{j_w} B$ and thus $|\rho^{-1}(z)| \le \frac{1}{1 - |\nabla \tau|_\infty }2^{j_w} \le 2 2^{j_w}$ 
	($\tau(0) =0$ by that $\tau$ is odd, and then $|\tau(\rho^{-1}(z))| \le |\nabla \tau|_\infty |\rho^{-1}(z)|$), 
	the r.h.s of \eqref{eq:xu0-xu1-bound} 
	$\le  2(2^{j_w} + 2^{j_f})  |\nabla \tau|_\infty  \|\nabla x\|_1$, 
	which proves \eqref{eq:uniform-bound-x-diff}.
\end{proof}

%
\begin{proof}[Proof of Theorem \color{red}{1}]
	We need a slightly generalized form of Lemma \ref{lemma:commute}, 
	which inserts multiple plain convolutional layers between $\ast w$ and $\ast f$,
	presented in Lemma \ref{lemma:comm+plus}.
	
	Under the setting of the theorem, in the generative CNNs,
	\begin{align}
		X_s = \sigma( \cdots \sigma( h\ast w_s^{(-L)} + b_s^{(-L)} ) \cdots  \ast w_s^{(-1)} +  b_s^{(-1)})\\
		X_t = \sigma( \cdots \sigma( h\ast w_t^{(-L)} + b_t^{(-L)} ) \cdots  \ast w_t^{(-1)} +  b_t^{(-1)})
	\end{align}
	where $w_t^{(l)}$ and $b_t^{(l)}$ are defined by, $l=-L,\cdots,-1$,
	\begin{equation}\label{eq:def-wt-bt-gen}
		w_t^{(l)} = D_l w_s^{(l)},
		\quad
		\tilde{x}_0^{(l)} \ast w_t^{(l)} + b_t^{(l)} = \tilde{x}_0^{(l)}\ast w_s^{(l)} + b_s^{(l)}.
	\end{equation}
	The notation $\tilde{x}^{(l)}$ stands for the $l$-th layer output in the target net from the input in the bottom (($-L$)-th) layer
	as $\tilde{x}^{(-L)} = h$, $\tilde{x}^{(0)} = X_t$,
	and $\tilde{x}_0^{(l)}$ for that from zero input in the bottom.
	In the feature CNNs, the $L$-th layer outputs are 
	\begin{align}
		F_s & = \sigma( \cdots \sigma( X_s \ast w_s^{(1)} + b_s^{(1)} ) \cdots  \ast w_s^{(L)} +  b_s^{(L)})  \\
		F_t  & = \sigma( \cdots \sigma( X_t \ast w_t^{(1)} + b_t^{(1)} ) \cdots  \ast w_t^{(L)} +  b_t^{(L)})  
	\end{align}
	where for $l=1,\cdots,L$,
	\[
	w_t^{(l)} = D_l w_s^{(l)},
	\quad
	b_t^{(l)} =  b_s^{(l)}.
	\]
	
	The proof is by applying Lemma \ref{lemma:comm+plus} recursively to the pair of layers indexed by $l$ and $-l$,
	from $l=1$ to $L$.  
	Denote $w_s^{(l)}$ by $w^{(l)}$, then $w_t^{(l)} = D_l w^{(l)}$, where $D_{-l} = D_l = D_{\tau_l}$, $l=1,\cdots, L$.
	We also denote $b_s^{(l)}$ by $b^{(l)}$ and keep notation $b_t^{(l)}$ for negative $l$. 
	
	First, $l=1$, in the target net,
	\[
	\tilde{x}^{(1)}:= \sigma( \sigma( \tilde{x}^{(-1)} \ast D_1 w^{(-1)} +  b_t^{(-1)}) \ast D_1 w^{(1)} + b^{(1)} )
	\]
	Use the centering $\tilde{x}_c^{(-1)} := \tilde{x}^{(-1)}  - \tilde{x}^{(-1)}_0 $,
	it can be written as 
	\begin{align}
		\tilde{x}^{(1)} 
		& = \sigma( \sigma(  \tilde{x}_c^{(-1)} \ast D_1 w^{(-1)} + \tilde{x}_0^{(-1)} \ast D_1 w^{(-1)} +  b_t^{(-1)}) \ast D_1 w^{(1)} + b^{(1)} ) \\
		& = \sigma( \sigma(  \tilde{x}_c^{(-1)} \ast D_1 w^{(-1)} + ( \tilde{x}_0^{(-1)} \ast w^{(-1)} +  b^{(-1)}) ) \ast D_1 w^{(1)} + b^{(1)} )
		\text{~~(by \eqref{eq:def-wt-bt-gen})}
	\end{align}
	Applying Lemma \ref{lemma:comm+plus} 
	(or Lemma \ref{lemma:commute} for this case),
	taking $\tilde{x}_0^{(-1)} \ast w^{(-1)} +  b^{(-1)}$ as the effective ``$b$'', 
	we have that  (using the non-expansiveness of $\sigma$ to take $r$ outside the last $\sigma$)
	\begin{align}
		\tilde{x}^{(1)}  
		& =  \sigma( \sigma(  \tilde{x}_c^{(-1)} \ast w^{(-1)} + \tilde{x}_0^{(-1)} \ast w^{(-1)} +  b^{(-1)}) \ast w^{(1)} + b^{(1)} ) + r^{(1)} \\
		& = \sigma( \sigma(  \tilde{x}^{(-1)} \ast w^{(-1)} +  b^{(-1)}) \ast w^{(1)} + b^{(1)} ) + r^{(1)}   \\
		& := \hat{x}^{(1)}  +  r^{(1)}   
		\label{eq:def-hatx1}
	\end{align}
	where, since $w^{(-1)}$, $w^{(1)}$ are supported on $2^{j_1} B$, 
	\begin{equation}\label{eq:bound-r1}
		\|r^{(1)}\|_1 
		\le 
		4 \varepsilon \left\{
		2^{j_1}  \| \nabla \tilde{x}_c^{(-1)} \|_1  + 2  \|\tilde{x}_c^{(-1)}\|_1
		\right\}.
	\end{equation}
	
	Next,
	\begin{align}
		\tilde{x}^{(2)}
		& := \sigma(     \tilde{x}^{(1)} 
		\ast D_2 w^{(2)} + b^{(2)} ) \\
		& = \sigma(  (\hat{x}^{(1)} + r^{(1)}  )
		\ast D_2 w^{(2)} + b^{(2)} )
		\text{~~(by \eqref{eq:def-hatx1})} \\
		& = \sigma(  \hat{x}^{(1)}   
		\ast D_2 w^{(2)} + b^{(2)} ) +{ r^{(1)}}' 
		\label{eq:tildex2-1}
	\end{align}
	where $\| {r^{(1)}}'\|_1 \le \| r^{(1)}\|_1$ and observe the same bound as  \eqref{eq:bound-r1},
	since neither $\ast w_t^{(2)}$ (Lemma \ref{lemma:non-expansive}(i)) nor applying $\sigma$ with bias expands the 1-norm. 
	Using the brief notation $\sigma_l$ to denote the non-linear mapping with biases $b^{(l)}$,
	consider
	\begin{align*}
		\sigma_2( \hat{x}^{(1)}    \ast D_2 w^{(2)}  )
		& =
		\sigma_2( \sigma_{1}(
		\sigma_{-1}( 
		\tilde{x}^{(-1)} \ast w^{(-1)} ) 
		\ast w^{(1)}  )  
		\ast D_2 w^{(2)} )\\
		& =
		\sigma_2(
		\sigma_{1}( 
		\sigma_{-1}( 
		\sigma( \tilde{x}^{(-2)} \ast D_2 w^{(-2)} + b_t^{(-2)})  
		\ast w^{(-1)} ) 
		\ast w^{(1)}  )  
		\ast D_2 w^{(2)} )\\
		& =
		\sigma_2( \sigma_{1}( 
		\sigma_{-1}( 
		\sigma( 
		\tilde{x}_c^{(-2)} \ast D_2 w^{(-2)} +  \tilde{x}_0^{(-2)} \ast w^{(-2)} + b^{(-2)}  )  \\
		& ~~~~~~~~~ 	   
		\ast w^{(-1)} ) 
		\ast w^{(1)} )  
		\ast D_2 w^{(2)}),	  \text{~~(by \eqref{eq:def-wt-bt-gen})}
	\end{align*}
	by Lemma \ref{lemma:comm+plus}, it equals (using the non-expansiveness of $\sigma_2$ to take $r^{(2)}$ outside)
	\begin{align*}
		& ~~~\sigma_2( 
		\sigma_{1}( 
		\sigma_{-1}( 
		\sigma( 
		\tilde{x}_c^{(-2)} \ast w^{(-2)} +  \tilde{x}_0^{(-2)} \ast w^{(-2)} + b^{(-2)}  )    
		\ast w^{(-1)} ) 
		\ast w^{(1)} )  
		\ast w^{(2)} ) + r^{(2)} \\
		& = 	  
		\sigma_2(
		\sigma_{1}( 
		\sigma_{-1}( 
		\sigma( \tilde{x}^{(-2)} \ast w^{(-2)} + b^{(-2)}  )    
		\ast w^{(-1)} ) 
		\ast w^{(1)} )  
		\ast w^{(2)} ) + r^{(2)} \\
		& := \hat{x}^{(2)}  +  r^{(2)}   
		\label{eq:def-hatx2}	  
	\end{align*}
	where 
	\begin{equation}\label{eq:bound-r2}
		\|r^{(2)}\|_1 
		\le 
		4 \varepsilon \left\{
		2^{j_2}  \| \nabla \tilde{x}_c^{(-2)} \|_1  + 2  \|\tilde{x}_c^{(-2)}\|_1
		\right\}.
	\end{equation}
	Inserting back to \eqref{eq:tildex2-1},
	\[
	\tilde{x}^{(2)}
	= 
	\hat{x}^{(2)}
	+ {r^{(1)}}'  + r^{(2)}   
	\]
	thus $\| \tilde{x}^{(2)} - \hat{x}^{(2)}  \|_1$ is bounded by the sum of \eqref{eq:bound-r1} and \eqref{eq:bound-r2}.
	
	Continue the process,
	$ \hat{x}^{(l)} $ denotes the $l$-th layer output in the source CNN (after $l$ times correction in the target CNN)
	by feeding 
	$ \tilde{x}^{(-l-1)}$ from the ($-l$)-th layer,
	where $\tilde{x}^{(-l-1)}$ is the output in the (un-corrected) generative target CNN after the first $(L-l)$ layers.
	By that $\tilde{x}^{(-L)} = x^{(-L)} = h$, 
	and that $F_t=  \tilde{x}^{(L)}$, $F_s=  x^{(L)}$, repeating the argument $L$ times gives that
	\[
	\|F_s - F_t\|_1 
	\le 
	4 \varepsilon 
	\sum_{l=1}^L ( 2^{j_l}  \| \nabla \tilde{x}_c^{(-l)} \|_1  + 2  \|\tilde{x}_c^{(-l)}\|_1),
	\]
	and when $(I_d-\rho_l)$ are rigid motions, the 2nd term for each $l$ vanishes. 
	
	We claim that
	
	Claim 3. For $l=-L, \cdots, -1$, $  \| \nabla \tilde{x}_c^{(l)} \|_1 \le   \| \nabla h \|_1$, and $  \| \tilde{x}_c^{(l)} \|_1 \le   \| h \|_1$.
	
	which suffices to prove the theorem.
	
	Proof of Claim 3:
	No that in the bottom layer $\tilde{x}_c^{(-L)} =\tilde{x}^{(-L)}  = h$. 
	For $l=-L+1,\cdots, -1$,
	\begin{align*}
		\| \tilde{x}_c^{(l)} \|_1 
		& = \| \tilde{x}^{(l)} - \tilde{x}_0^{(l)}\|_1 \\
		& = \| \sigma_l(  \tilde{x}^{(l-1)}\ast  w_t^{(l-1)}) - \sigma_l(  \tilde{x}_0^{(l)} \ast w_t^{(l-1)} )\|_1 \\
		& \le \|  \tilde{x}^{(l-1)}\ast  w_t^{(l-1)}  -   \tilde{x}_0^{(l)} \ast w_t^{(l-1)} \|_1  
		~~~ \text{ (by that $ \sigma_l$ non-expansive)} \\
		& \le 	\|  \tilde{x}^{(l-1)} -   \tilde{x}_0^{(l)} \|_1 
		~~~ \text{ (by that $ \| w_t^{(l-1)} \|_1 \le 1$ and Lemma \ref{lemma:non-expansive}(i))} \\
		& = 	\| \tilde{x}_c^{(l-1)} \|_1.
	\end{align*}
	Recursing the inequality gives that $\| \tilde{x}_c^{(l)} \|_1  \le \| h \|_1 $.
	Similarly,
	\begin{align*}
		\| \nabla \tilde{x}_c^{(l)} \|_1  
		& =  \| \nabla \tilde{x}^{(l)} \|_1   
		= \text{TV}[ \sigma_l( \tilde{x}^{(l-1)} \ast w_t^{(l-1)} ) ] \\
		& \le \text{TV}[ \tilde{x}^{(l-1)} \ast w_t^{(l-1)}  ] 
		~~~ \text{ (by that $ \sigma_l$ does not increase total variation)} \\
		&  = \| \nabla (  \tilde{x}^{(l-1)} \ast w_t^{(l-1)} )  \|_1 \\
		&  \le  \| \nabla  \tilde{x}^{(l-1)} \|_1	
		= \| \nabla  \tilde{x}_c^{(l-1)} \|_1,	
		~~~ \text{ (by that $ \| w_t^{(l-1)} \|_1 \le 1$ and Lemma \ref{lemma:non-expansive}(ii))} 
	\end{align*}
	and thus $\| \nabla \tilde{x}_c^{(l)} \|_1  \le \| \nabla h \|_1 $. This proves Claim 3.
\end{proof}

\begin{lemma}\label{lemma:comm+plus}
	Suppose filters $w$, $f_1, \cdots, f_m$, $f$ satisfy that the 1-norm are all bounded by 1,
	and $w$  and $f$ are supported on $2^{j}B$.
	The sequence of $\sigma_l$, 
	denoting non-linear function with bias,
	for $l=0,\cdots,m$ are non-expansive.
	$D_\tau$ is a spatial transform where $\tau$ is odd and $|\nabla \tau|_\infty \le \varepsilon < \frac{1}{5}$.
	Then 
	\[
	\sigma_m ( \cdots \sigma_1( \sigma_0( x \ast D_\tau w) \ast f_1) \cdots \ast f_m) \ast f 
	\]
	approximates
	\[
	\sigma_m ( \cdots \sigma_1( \sigma_0( x \ast  w) \ast f_1) \cdots \ast f_m) \ast D_\tau^{-1} f 
	\]
	up to an error whose 1-norm is bounded by
	\[
	4 \varepsilon \left\{
	2^{j}  \| \nabla x\|_1  + 2  \|x\|_1
	\right\},
	\]
	and the second term vanishes if $(I_d - \tau)$ is a rigid motion. 	
\end{lemma}

\begin{proof}[Proof of Lemma \ref{lemma:comm+plus}]
	The proof uses the same technique as in the proof of Lemma \ref{lemma:commute}.
	Omitting subscript $\R^2$ in the integral, let
	\begin{align*}
		y_1(u) 
		& = \int 
		\sigma_m(\int \cdots
		\sigma_1(\int
		\sigma_0(
		\int x( u+ v_1 + \cdots + v_m + v -  \rho^{-1}(z) ) w(z) |J\rho^{-1}| dz) \\
		& ~~~~~~~~~~~~~~~~~~~
		f(-v_1) dv_1) \cdots f_m(-v_m)dv_m) f(-v) dv, \\
		\hat{y}_1(u) 
		& = \int 
		\sigma_m(\int \cdots
		\sigma_1(\int
		\sigma_0(
		\int x( u+ v_1 + \cdots + v_m + v -  \rho^{-1}(z) ) w(z) dz) \\
		& ~~~~~~~~~~~~~~~~~~~
		f(-v_1) dv_1) \cdots f_m(-v_m)dv_m) f(-v) dv.
	\end{align*}
	By Fact 1, 
	that $\sigma_j$ are all non-expansive
	and that the 1-norm of all the filters are bounded by 1, 
	\[
	\int | y_1(u)  - \hat{y}_1(u) | du \le 4 \varepsilon \|x\|_1. 
	\]
	Also,
	\begin{align*}
		y_2(u) 
		& = \int 
		\sigma_m(\int \cdots
		\sigma_1(\int
		\sigma_0(
		\int x( u+ v_1 + \cdots + v_m + \rho(v) -  z ) w(z) dz) \\
		& ~~~~~~~~~~~~~~~~~~~
		f(-v_1) dv_1) \cdots f_m(-v_m)dv_m) f(-v) |J\rho| dv, \\
		\hat{y}_2(u) 
		& = \int 
		\sigma_m(\int \cdots
		\sigma_1(\int
		\sigma_0(
		\int x( u+ v_1 + \cdots + v_m + \rho(v) -  z ) w(z) dz) \\
		& ~~~~~~~~~~~~~~~~~~~
		f(-v_1) dv_1) \cdots f_m(-v_m)dv_m) f(-v) dv.
	\end{align*}
	Similarly,
	\[
	\int | y_2(u)  - \hat{y}_2(u) | du \le 4 \varepsilon \|x\|_1. 
	\]
	Same as before, with $\rho$ being a rigid motion, $\| y_1  - \hat{y}_1 \|$ and $\| y_2  - \hat{y}_2 \|$ are both zero.
	
	It remains to bound $\|\hat{y}_1 - \hat{y}_2 \|_1$. Observe that
	\begin{align}
		& \int | \hat{y}_1(u) - \hat{y}_2 (u)|du
		\le 
		\int \cdots \int dv |f(-v)| dv_m |f(-v_m)| \cdots dv_1 |f(-v_1)| dz |w(z)|  \nonumber \\
		& ~~~~~~~~
		\int du | x( u+ v_1 + \cdots + v_m + v -  \rho^{-1}(z) ) - x( u+ v_1 + \cdots + v_m + \rho(v) -  z  ) |,
		\label{eq:haty1-haty2}
	\end{align}
	and similarly as in proving Lemma \ref{lemma:commute},
	one can verify that for any fixed $v_1, \cdots, v_m$, $v$, $z$, 
	\begin{align*}
		& \int  | x( u+ v_1 + \cdots + v_m + v -  \rho^{-1}(z) ) - x( u+ v_1 + \cdots + v_m + \rho(v) -  z  ) | du \\
		&  ~~~~~~~~
		\le  \| \nabla x\|_1 |\nabla \tau |_\infty |v - \rho^{-1}(z)|
		\le \varepsilon 2 (2^j+ 2^j) \| \nabla x\|_1.
	\end{align*}
	Inserting back to \eqref{eq:haty1-haty2}, and again by that the 1-norm of all the filters are bounded by 1, 
	we have that $\| \hat{y}_1  - \hat{y}_2\|_1  \le 4 \varepsilon 2^j \| \nabla x\|_1$.
\end{proof}

\begin{lemma}\label{lemma:non-expansive}
	Let $x$ and $w$ be smooth and compactly supported on $\R^2$,
	then
	
	(i) $\|x \ast w\|_1 \le \|x\|_1 \|w\|_1$.
	
	(ii) $\| \nabla( x \ast w) \|_1 \le \| \nabla x \|_1 \|w\|_1$.
	
\end{lemma}
\begin{proof}[Proof of Lemma \ref{lemma:non-expansive}]
	For (i), 
	\[
	\|x \ast w\|_1
	=
	\int_{\R^2} | \int_{\R^2} x(u-v) w(v) dv| du
	\le
	\int_{\R^2}  \int_{\R^2} |x(u-v)| |w(v)| du dv 
	= \|x\|_1 \|w\|_1.
	\]
	For (ii),
	\begin{align*}
		\| \nabla( x \ast w) \|_1
		& = \int_{\R^2} | \nabla_u ( \int_{\R^2} x(u-v) w(v) dv ) | du \\
		& =  \int_{\R^2} |  \int_{\R^2} \nabla_u  x(u-v) w(v) dv  | du  \\
		& \le  \int_{\R^2} \int_{\R^2} | \nabla_u  x(u-v)|  |w(v)| du dv  \\
		& =  \| \nabla x \|_1 \|w\|_1.
	\end{align*}
\end{proof}

\end{document}